\documentclass{article}

\usepackage{microtype}
\usepackage{graphicx}
\usepackage{booktabs} 

\usepackage[accepted]{icml2023}

\usepackage{natbib}
\defcitealias{TanZhuGuo2021}{\({}^{\ast}\)}
\defcitealias{BauerMarquartBoetiusLeueEtAl2021}{\({}^{\star}\)}

\usepackage{amsmath}
\usepackage{amssymb}
\usepackage{mathtools}
\usepackage{amsthm}

\usepackage{hyperref}

\theoremstyle{plain}
\newtheorem{theorem}{Theorem}[section]
\newtheorem{proposition}[theorem]{Proposition}
\newtheorem{lemma}[theorem]{Lemma}

\theoremstyle{definition}
\newtheorem{defn}[theorem]{Definition}

\theoremstyle{remark}
\newtheorem{remark}[theorem]{Remark}
\theoremstyle{definition}
\newtheorem{example}{Example}

\usepackage[textsize=tiny]{todonotes}

\usepackage{algorithm}
\usepackage{algorithmic}
\renewcommand{\algorithmiccomment}[1]{\textit{// #1}}

\usepackage{caption}
\usepackage{subcaption}
\usepackage{bm}
\usepackage{wasysym}
\usepackage{scalerel}[2016/12/29]  
\usepackage{siunitx}
\newcommand\restr[2]{{
    {\left.#1\right|}_{#2}
}}
\newcommand\transp[1]{{
    #1^{\mkern-1.5mu\mathsf{T}}
}}

\newcommand{\Reals}{\mathbb{R}}

\newcommand{\Nats}{\mathbb{N}}

\newcommand{\ReLU}[1]{{\left[#1\right]}^+}
\NewDocumentEnvironment{proofidea}{}{%
  \begin{proof}[Proof Idea]
}{%
  \end{proof}
}
\renewcommand{\vec}[1]{\mathbf{#1}}
\newcommand{\varvec}[1]{\boldsymbol{#1}}
\newcommand{\mat}[1]{\mathbf{#1}}

\newcommand{\vertices}[1]{\mathrm{vert}{\left(#1\right)}}
\newcommand{\cxremove}{{C\!\!\:R}}
\newcommand{\scenprob}{{S\!\!\:P}}

\NewDocumentCommand{\NN}{O{\varvec{\theta}}}{%
  \mathrm{net}_{#1}
}
\newcommand{\fSAT}{f_{\mathrm{Sat}}}
\newcommand{\exsuccess}{\scaleobj{1.2}{\boldsymbol{\checkmark}}}
\newcommand{\extimeout}{\scaleobj{1.05}{\bell}}

\icmltitlerunning{%
  A Robust Optimisation Perspective on Counterexample-Guided Repair of Neural Networks
}

\begin{document}

\twocolumn[
\icmltitle{%
  A Robust Optimisation Perspective on \\
  Counterexample-Guided Repair of Neural Networks
}
\icmltitlerunning{%
  A Robust Optimisation Perspective on
  Counterexample-Guided Repair of Neural Networks
}



\icmlsetsymbol{equal}{*}

\begin{icmlauthorlist}
\icmlauthor{David Boetius}{unikn}
\icmlauthor{Stefan Leue}{unikn}
\icmlauthor{Tobias Sutter}{unikn}
\end{icmlauthorlist}

\icmlaffiliation{unikn}{%
  Department of Computer and Information Science, 
  University of Konstanz, Konstanz, Baden-Württemberg, Germany%
}

\icmlcorrespondingauthor{David Boetius}{david.boetius@uni-konstanz.de}

\icmlkeywords{%
  Machine Learning, ICML, Safe Machine Learning, Trustworthy Machine Learning,
  Neural Networks, Deep Learning, Neural Network Repair, 
  Counterexample-Guided Repair, Neural Network Verification, 
  Verification, Repair, Formal Methods in Machine Learning, 
  Formal Methods in Deep Learning, Formal Methods,
}

\vskip 0.3in
]



\printAffiliationsAndNotice{}  

\begin{abstract}
  Counterexample-guided repair aims at creating neural networks 
  with mathematical safety guarantees, facilitating 
  the application of neural networks in safety-critical domains.
  However, whether counterexample-guided repair
  is guaranteed to terminate remains an open question.
  We approach this question by showing that counterexample-guided repair 
  can be viewed as a robust optimisation algorithm.
  While termination guarantees for neural network repair itself
  remain beyond our reach, we prove termination for 
  more restrained machine learning models
  and disprove termination in a general setting.
  We empirically study the practical implications of our theoretical results,
  demonstrating the suitability of common verifiers and falsifiers for repair
  despite a disadvantageous theoretical result.
  Additionally, we use our theoretical insights to devise a novel
  algorithm for repairing linear regression models 
  based on quadratic programming, surpassing existing approaches.
\end{abstract}

\section{Introduction}
The success of artificial neural networks in such diverse domains as 
image recognition~\citep{LeCunBottouBengioEtAl1998},
natural language processing~\citep{BrownMannRyderEtAl2020},
predicting protein folding~\citep{SeniorEvansJumperEtAl2020},
and designing novel algorithms~\citep{FawziBalogHuangEtAl2022} sparks
interest in applying them to more demanding tasks, including applications in
safety-critical domains.
Neural networks are proposed to be used for medical 
diagnosis~\citep{AmatoLopezPenaMendezEtAl2013}, 
autonomous aircraft control~\citep{Jorgensen1997,JulianKochenderferOwen2018},
and self-driving cars~\citep{BojarskiTestaDworakowskiEtAl2016}. 
Since a malfunctioning of such systems can threaten lives or cause environmental disaster,
we require mathematical guarantees on the correct functioning 
of the neural networks involved.
Formal methods, including verification and repair, allow obtaining such guarantees~\citep{PulinaTacchella2010}. 
As the inner workings of neural networks are opaque to human engineers, 
automated repair is a vital component for creating safe neural networks. 

Alternating search for violations and removal of violations is 
a popular approach for repairing neural networks~\citep{%
  PulinaTacchella2010,%
  GoodfellowShlensSzegedy2015,%
  GuidottiLeofantePulinaEtAl2019,%
  DongSunWangEtAl2020,%
  GoldbergerKatzAdiEtAl2020,%
  SivaramanFarnadiMillsteinEtAl2020,%
  TanZhuGuo2021,%
  BauerMarquartBoetiusLeueEtAl2021%
}.
We study this approach under the name \emph{counterexample-guided repair}.
Counterexample-guided repair uses inputs for which a neural network
violates the specification (counterexamples) to iteratively refine the network until
the network satisfies the specification.
While empirical results demonstrate the ability of counterexample-guided repair 
to successfully repair neural networks~\citep{BauerMarquartBoetiusLeueEtAl2021}, 
a theoretical analysis of counterexample-guided repair is lacking. 

In this paper, we study counterexample-guided repair from the perspective of robust optimisation.
Viewing counterexample-guided repair as an algorithm for solving robust optimisation problems 
allows us to encircle neural network repair from two sides. 
On the one hand side, we are able to show termination and optimality for counterexample-guided repair of linear
regression models and linear classifiers, as well as single ReLU neurons. 
Coming from the other side, we disprove termination for
repairing ReLU networks to satisfy a 
specification with an unbounded input set. 
Additionally, we disprove termination of counterexample-guided repair when
using generic counterexamples without further qualifications, such as being most-violating.
While we could not address termination for the precise robust program
of neural network repair with specifications
having bounded input sets, 
such as \(L_\infty\) adversarial robustness
or the ACAS Xu safety properties~\citep{KatzBarrettDillEtAl2017}, 
our robust optimisation framework provides, 
for the first time to the best of our knowledge, fundamental insights
into the theoretical properties of counterexample-guided repair.

Our analysis establishes a theoretical limitation of repair with otherwise unqualified counterexamples
and suggests most-violating counterexamples as a replacement. 
We empirically investigate the practical consequences of these findings by comparing
\emph{early-exit} verifiers~---~verifiers that stop search on the first counterexample
they encounter~---~and optimal verifiers that produce most-violating counterexamples.
We complement this experiment by investigating the advantages of using 
falsifiers during repair~\citep{BauerMarquartBoetiusLeueEtAl2021},
which is another approach that leverages sub-optimal counterexamples.
These experiments do not reveal any practical limitations for repair using early-exit verifiers.
In fact, using an early-exit verifier consistently yields faster repair for
ACAS Xu networks~\citep{KatzBarrettDillEtAl2017} 
and an MNIST~\citep{LeCunBottouBengioEtAl1998} network, compared to using an optimal verifier. 
While the optimal verifier often allows performing fewer iterations, this advantage is offset by
its additional runtime cost most of the time.
Our experiments with falsifiers demonstrate that they can provide a significant runtime advantage
for neural network repair.

For repairing linear regression models, we use our theoretical insights to design an 
improved repair algorithm based on quadratic programming.
We compare this new algorithm with the Ouroboros~\citep{TanZhuGuo2021}
and SpecRepair~\citep{BauerMarquartBoetiusLeueEtAl2021} repair algorithms.
The new quadratic programming repair algorithm surpasses Ouroboros and SpecRepair,
illustrating the practical relevance of our theoretical results.

We highlight the following main contributions of this paper:
\begin{enumerate}
 \item We formalise neural network repair as a robust optimisation problem
   and, therefore, view counterexample-guided repair as a 
   robust optimisation algorithm.
 \item Using this theoretical framework, 
   we prove termination of counterexample-guided repair for
   more restrained problems than neural network repair and
   disprove termination in a more general setting.
 \item We empirically investigate the merits of using falsifiers and early-exit 
   verifiers during repair.
 \item Our theoretical insights into repairing linear regression models allow us to
   surpass existing approaches for repairing linear regression models
   using a new repair algorithm.
\end{enumerate}
%

\section{Related Work}\label{sec:related-work}
This paper is concerned with viewing neural network repair through 
the lens of robust optimisation. 
Neural network repair relies on neural network verification and 
can make use of neural network falsification.
Counterexample-guided repair is related to other 
counterexample-guided algorithms.
We introduce related work from these fields in this section.

\textbf{Neural Network Verification}\quad
We can verify neural networks, for example, using
Satisfiability Modulo Theories~(SMT) 
solving~\citep{KatzBarrettDillEtAl2017, Ehlers2017}
or Mixed Integer Linear Programming~(MILP)~\citep{TjengXiaoTedrake2019}.
Verification benefits from bounds computed using linear
relaxations~\citep{%
  Ehlers2017,%
  SinghGehrMirmanEtAl2018,%
  ZhangWengChenEtAl2018,%
  XuZhangWangEtAl2021%
}.
A particularly fruitful technique from MILP is branch and 
bound~\citep{%
  BunelTurkaslanTorrEtAl2018,PalmaBunelDesmaisonEtAl2021,WangZhangXuEtAl2021%
}.
Approaches that combine branch and bound
with multi-neuron linear relaxations~\citep{FerrariMuellerJovanovicEtAl2022}
or extend branch and bound using cutting planes~\citep{ZhangWangXuEtAl2022} 
form the current state-of-the-art~\citep{MuellerBrixBakEtAl2022}.
%

Falsifiers are designed to discover counterexamples fast at the cost of 
completeness~---~they can not prove specification satisfaction.
We view adversarial attacks as falsifiers for 
adversarial robustness specifications.
Falsifiers use generic local optimisation algorithms~\citep{%
  SzegedyZarembaSutskeverEtAl2014,%
  GoodfellowShlensSzegedy2015,%
  KurakinGoodfellowBengio2016,%
  MadryMakelovSchmidtEtAl2018%
}, global optimisation algorithms~\citep{%
  UesatoODonoghueKohliEtAl2018,BauerMarquartBoetiusLeueEtAl2021%
},
or specifically tailored search and optimisation techniques~\citep{%
  PapernotMcDanielJhaEtAl2015,ChenJordanWainwright2020%
}.

\textbf{Neural Network Repair}\quad
Neural network repair is concerned with modifying a neural network
such that it satisfies a formal specification.
Many approaches make use of the counterexample-guided repair algorithm
while utilising different counterexample-removal
algorithms.
The approaches range from augmenting the training 
set~\citep{%
  PulinaTacchella2010,%
  GoodfellowShlensSzegedy2015,%
  TanZhuGuo2021%
}, over specialised neural network 
architectures~\citep{GuidottiLeofantePulinaEtAl2019,DongSunWangEtAl2020},
and neural network training with 
constraints~\citep{BauerMarquartBoetiusLeueEtAl2021},
to using a verifier for computing network 
weights~\citep{GoldbergerKatzAdiEtAl2020}.
Counterexample-guided repair is also applied to support vector machines (SVMs) and 
linear regression 
models~\citep{GuidottiLeofanteTacchellaEtAl2019,TanZhuGuo2021}.
Using decoupled neural networks~\citep{SotoudehThakur2021b} provides 
optimality and termination guarantees for repair but is limited to 
two-dimensional input spaces.

When considering only adversarial robustness, an alternative to 
counterexample-guided repair is provably robust training.
Here, the applied techniques include 
interval arithmetic~\citep{GowalDvijothamStanforthEtAl2018},
semi-definite relaxations~\citep{RaghunathanSteinhardtLiang2018},
linear relaxations~\citep{MirmanGehrVechev2018},
and duality~\citep{WongKolter2018}.
%
%
Also specialised neural network architectures can increase 
the adversarial robustness of neural networks~\citep{CisseBojanowskiGraveEtAl2017,ZhangJiangHeEtAl2022}.

\textbf{Robust and Scenario Optimisation}\quad
Robust optimisation is, originally, a technique for dealing with data 
uncertainty~\citep{BenTalGhaouiNemirovski2009}.
Although robust optimisation problems are, in general, 
NP-hard~\citep{BenTalNemirovski1998},
polynomial-time methods exist in many~---typically 
convex~---~settings~\citep{BenTalGhaouiNemirovski2009}.
Scenario problems and chance-constrained problems 
relax robust problems.
Feasibility for a scenario problem can be linked to 
feasibility for a chance-constrained problem~\citep{CampiGaratti2008}
and even to a perturbed robust problem~\citep{EsfahaniSutterLygeros2015}.
In this paper, we consider the worst-case perspective
on robust optimisation~\citep{MutapcicBoyd2009}.

Loss functions based on robust optimisation can be used to train for
adversarial robustness and beyond~\citep{%
  MadryMakelovSchmidtEtAl2018,%
  WongKolter2018,%
  FischerBalunovicDrachslerCohenEtAl2019%
}.
We study a different robust optimisation formulation, 
where the specification is modelled as constraints.
%

\textbf{Counterexample-Guided Algorithms}\quad
Besides neural network repair, counterexample-guided approaches
are also used in
model checking~\citep{ClarkeGrumbergJhaEtAl2000},
program synthesis~\citep{SolarLezamaTancauBodikEtAl2006},
and beyond~\citep{%
  HenzingerJhalaMajumdar2003,%
  ReynoldsDetersKuncakEtAl2015,%
  NguyenAntonopoulosRuefEtAl2017%
}.
The notion of minimal counterexamples in 
program synthesis~\citep{JhaSeshia2014},
which is based on an ordering of the counterexamples,
encompasses our notion of most-violating counterexamples.
While related, the termination results for counterexample-guided
program synthesis~\citep{SolarLezamaTancauBodikEtAl2006} 
do not transfer to repairing neural networks,
as they only apply for finite input domains.
We study infinite input domains in this paper.

\section{Preliminaries and Problem Statement}\label{sec:prelim-problem-statement}
In this section, we introduce preliminaries on robust optimisation, 
neural networks, and neural network verification before progressing
to neural network repair, counterexample-guided repair, and the 
problem statement of our theoretical analysis. 


\subsection{Robust Optimisation}
We consider general \emph{robust optimisation problems} of the form
\begin{equation}
  P : \left\{\!
    \arraycolsep=2pt
    \begin{array}{cl}
      \underset{\vec{v}}{\text{minimise}} & f(\vec{v}) \\
      \text{subject to} & g{\left(\vec{v}, \vec{d}\right)} \geq 0 \quad \forall \vec{d} \in \mathcal{D}\\
                        & \vec{v} \in \mathcal{V},
    \end{array}
  \right.\label{eqn:p}
\end{equation}
where~\(\mathcal{V} \subseteq \Reals^v\),~\(\mathcal{D} \subseteq \Reals^d\)
and~\(f: \mathcal{V} \to \Reals\),~\(g: \mathcal{V} \times \mathcal{D} \to \Reals\).
Both~\(\mathcal{V}\) and~\(\mathcal{D}\) contain infinitely many elements. 
Therefore, a robust optimisation problem has infinitely many constraints.
The \emph{variable domain}~\(\mathcal{V}\) defines eligible values for 
the \emph{optimisation variable}~\(\vec{v}\).
The set~\(\mathcal{D}\) may contain, for example, all inputs for which a specification needs to hold.
In this example,~\(g\) captures whether the specification is satisfied for a concrete input. 
Elaborating this example leads to neural network repair, 
which we introduce in Section~\ref{sec:prelim-nn-repair}.
 
A \emph{scenario optimisation problem} relaxes a robust optimisation problem by
replacing the infinitely many constraints of~\(P\) with a finite selection.
For~\(\vec{d}^{(i)} \in \mathcal{D}\),~\(i \in \{1, \ldots, N\}\),~\(N \in \Nats\),
the scenario optimisation problem is
\begin{equation}
  \scenprob : \left\{\!
    \arraycolsep=2pt
    \begin{array}{cl}
      \underset{\vec{v}}{\text{minimise}} & f(\vec{v}) \\
      \text{subject to} & g{\left(\vec{v}, \vec{d}^{(i)}\right)} \geq 0 \quad \forall i \in \{1, \ldots, N\} \\
                        & \vec{v} \in \mathcal{V}.
    \end{array}
  \right.\label{eqn:sp}
\end{equation}
The counterexample-guided repair algorithm that we study in this paper uses a sequence of scenario optimisation problems 
to solve a robust optimisation problem.

\subsection{Neural Networks}
A neural network~\(\NN{}: \Reals^n \to \Reals^m\), with 
parameters~\(\varvec{\theta} \in \Reals^p\) 
is a function composition of affine transformations
and non-linear activations.
For our theoretical analysis, it suffices to consider
fully-connected neural networks (FCNN).
Our experiments in Section~\ref{sec:experiments} also use 
convolutional neural networks (CNN).
We refer to~\citet{GoodfellowBengioCourville2016} for an introduction to CNNs.
An FCNN with~\(L\) \emph{hidden layers} is a chain of affine functions and activation functions
\begin{equation}
  \NN{} = h^{(L+1)} \circ \sigma^{(L)} \circ h^{(L)} \circ \cdots \circ \sigma^{(1)} \circ h^{(1)},
\end{equation}
where~\(h^{(i)}: \Reals^{n_{i-1}} \to \Reals^{n_i}\) 
and~\(\sigma^{(i)}: \Reals^{n_i} \to \Reals^{n_i}\) 
with~\(n_i \in \Nats\) for~\(i \in \{0, \ldots, L+1\}\)
and, specifically,~\(n_0 = n\) and~\(n_{L+1} = m\).
Each~\(h^{(i)}\) is an affine function, called an \emph{affine layer}.
It computes~\(h^{(i)}(\vec{z}) = \mat{W}^{(i)} \vec{z} + \vec{b}^{(i)}\)
with \emph{weight matrix}~\(\mat{W}^{(i)} \in \Reals^{n_i \times n_{i-1}}\) 
and \emph{bias vector}~\(\vec{b}^{(i)} \in \Reals^{n_i}\).
Stacked into one large vector, the weights and biases of all affine layers
are the \emph{parameters}~\(\varvec{\theta}\) of the FCNN.\@
An \emph{activation layer}~\(\sigma^{(i)}\) applies a non-linear function, such as 
ReLU~\(\ReLU{z} = \max(0, z)\) or the sigmoid 
function~\(\sigma(z) = \frac{1}{1 + e^{-z}}\) in an element-wise fashion.

\subsection{Neural Network Verification}\label{sec:nn-verif}
Neural network verification is concerned with automatically
proving that a neural network satisfies a formal specification.

\begin{defn}[Specification]
  A \emph{specification}~\(\Phi = \{\varphi_1, \ldots, \varphi_S\}\) is a set
  of \emph{properties}~\(\varphi_i\).
  A \emph{property}~\(\varphi = (\mathcal{X}_{\varphi}, \mathcal{Y}_{\varphi})\) 
  is a tuple of an \emph{input set}~\(\mathcal{X}_{\varphi} \subseteq \Reals^n\)
  and an \emph{output set}~\(\mathcal{Y}_{\varphi} \subseteq \Reals^m\).

  We write~\(\NN \vDash \Phi\) when a neural network~\(\NN: \Reals^n \to \Reals^m\)
  \emph{satisfies} a specification~\(\Phi\).
  Specifically,
  \begin{subequations}
    \begin{align}
      \NN \vDash \Phi &\Leftrightarrow \forall \varphi \in \Phi: \NN \vDash \varphi \\
      \NN \vDash \varphi &\Leftrightarrow 
      \forall \vec{x} \in \mathcal{X}_\varphi: \NN(\vec{x}) \in \mathcal{Y}_\varphi.
    \end{align}
  \end{subequations}
\end{defn}
Example specifications, such as~\(L_\infty\) adversarial robustness or
an ACAS Xu safety specification~\citep{KatzBarrettDillEtAl2017} 
are defined in Appendix~\ref{sec:specs}.

\emph{Counterexamples} are at the core of the counterexample-guided
repair algorithm that we study in this paper.
\begin{defn}[Counterexample]
  An input~\(\vec{x} \in \mathcal{X}_\varphi\) for 
  which a neural network~\(\NN\)
  violates a property~\(\varphi\) is called a \emph{counterexample}.
\end{defn}

To define verification as an optimisation problem, we introduce \emph{satisfaction functions}.
A satisfaction function quantifies the satisfaction or violation of a property
regarding the output set. 
Definition~\ref{dfn:verification-problem} introduces the 
verification problem, also taking the input set of a property into account.
\begin{defn}[Satisfaction Function]
  A function~\(\fSAT{}: \Reals^m \to \Reals\) is a \emph{satisfaction function} of a 
  property~\(\varphi = (\mathcal{X}_\varphi, \mathcal{Y}_\varphi)\) if
  \(\vec{y} \in \mathcal{Y}_\varphi \Leftrightarrow \fSAT(\vec{y}) \geq 0\).
\end{defn}
\begin{defn}[Verification Problem]\label{dfn:verification-problem}
  The \emph{verification problem} for a property~\(\varphi = (\mathcal{X}_\varphi, \mathcal{Y}_\varphi)\) 
  and a neural network~\(\NN\) is
  \begin{equation}
    V : \fSAT^\ast = \left\{\!
      \arraycolsep=2pt
      \begin{array}{cl}
        \underset{\vec{x}}{\text{minimise}} & \fSAT(\NN(\vec{x})) \\
        \text{subject to} & \vec{x} \in \mathcal{X_\varphi}.
      \end{array}
    \right.\label{eqn:v}
  \end{equation}
\end{defn}
We call a specification a \emph{linear specification} when its properties
have a closed convex polytope as an input set 
and an affine satisfaction function.
Appendix~\ref{sec:specs} contains several examples of 
satisfaction functions from SpecRepair~\citep{BauerMarquartBoetiusLeueEtAl2021}.
The following Proposition follows directly from the definition of a satisfaction function.
\begin{proposition}
  A neural network~\(\NN\) satisfies the property~\(\varphi\) if and only if 
  the minimum of the verification problem~\(V\) (\(\fSAT^\ast\)) 
  is non-negative.
\end{proposition}

We now consider \emph{counterexample searchers} that evaluate the satisfaction function
for concrete inputs to compute an upper bound on 
the minimum of the verification problem~\(V\).
Such tools can disprove specification satisfaction by discovering a counterexample.
They can also prove specification satisfaction when they are \emph{sound} and \emph{complete}.
\begin{defn}[Soundness and Completeness]
  We call a counterexample searcher \emph{sound} if it computes valid upper bounds
  and \emph{complete} if it is guaranteed to find a counterexample
  whenever a counterexample exists.
\end{defn}
\begin{defn}[Verifiers and Falsifiers]
  We call a sound and complete counterexample searcher a \emph{verifier}.
  A sound but incomplete counterexample searcher is a \emph{falsifier}.
\end{defn}
\begin{defn}[Optimal and Early-Exit Verifiers]\label{dfn:optimal-early-exit-verifiers}
  An \emph{optimal} verifier is a verifier that always produces a global minimiser of the verification 
  problem~---~a \emph{most-violating counterexample}.
  An \emph{early-exit} verifier aborts on the first counterexample it encounters.
  It provides any counterexample, without further qualifications.
\end{defn}
We can perform verification through global minimisation of the verification problem
from Equation~\eqref{eqn:v}.
For ReLU-activated neural networks, global minimisation is possible, for example,
using Mixed Integer Linear Programming~(MILP)~\citep{ChengNuehrenbergRuess2017,LomuscioMaganti2017}.
Falsifiers may perform local optimisation using projected gradient
descent~(PGD)~\citep{KurakinGoodfellowBengio2016,MadryMakelovSchmidtEtAl2018} 
to become sound but not complete.
We name the approach of using PGD for falsification \emph{BIM}, abbreviating the 
name \emph{Basic Iterative Method} used 
by~\citet{KurakinGoodfellowBengio2016}.

\subsection{Neural Network Repair}\label{sec:prelim-nn-repair}
Neural network repair means modifying a trained neural network so that it satisfies a 
specification it would otherwise violate.
While the primary goal of repair is satisfying the specification, the key secondary goal 
is that the repaired neural network still performs well on the intended task.
This secondary goal can be captured using a performance measure, such as the 
training loss function~\citep{BauerMarquartBoetiusLeueEtAl2021} 
or the distance between the modified and the original network parameters~\citep{GoldbergerKatzAdiEtAl2020}.

\begin{defn}[Repair Problem]\label{dfn:repair-problem}
  Given a neural network~\(\NN\), a property~\(\varphi = (\mathcal{X}_\varphi, \mathcal{Y}_\varphi)\) and
  a performance measure~\(J: \Reals^p \to \Reals\), repair translates to solving the \emph{repair problem}
  \begin{equation}
    R : \left\{\!
      \arraycolsep=2pt
      \begin{array}{cl}
        \underset{\varvec{\theta} \in \Reals^p}{\text{minimise}} & J(\varvec{\theta}) \\
        \text{subject to} & \fSAT(\NN(\vec{x})) \geq 0 \quad \forall \vec{x} \in \mathcal{X}_\varphi.
      \end{array}
    \right.\label{eqn:r}
  \end{equation}
\end{defn}
The repair problem~\(R\) is an instance of a robust optimisation problem 
as defined in Equation~\eqref{eqn:p}.
Checking whether a parameter~\(\varvec{\theta}\) is feasible 
for~\(R\) corresponds to verification.
In particular, we can equivalently reformulate~\(R\) using 
the verification problem's minimum~\(\fSAT^\ast\) from Equation~\eqref{eqn:v} as
\begin{equation}
  R' : \left\{\!
    \arraycolsep=2pt
    \begin{array}{cl}
      \underset{\varvec{\theta} \in \Reals^p}{\text{minimise}} & J(\varvec{\theta}) \\
      \text{subject to} & \fSAT^\ast \geq 0.
    \end{array}
  \right.\label{eqn:r-alt}
\end{equation}

We stress several characteristics of the repair problem that we relax or strengthen in
Section~\ref{sec:theory}.
First of all,~\(\NN\) is a neural network and we repair all parameters~\(\varvec{\theta}\) 
of the network jointly. 
Practically,~\(\NN\) is a ReLU-activated FCNN or CNN, as these are the models most verifiers
support.
For typical specifications, such as~\(L_\infty\) adversarial robustness or 
the ACAS Xu safety specifications~\citep{KatzBarrettDillEtAl2017},
the property input set~\(\mathcal{X}_\varphi\) is a hyper-rectangle.
Hyper-rectangles are closed convex polytopes and, therefore, bounded.

\subsection{Counterexample-Guided Repair}\label{sec:cgr}
In the previous section, we have seen that the repair problem includes the verification problem
as a sub-problem.
Using this insight, a natural approach to tackle the repair problem is to iterate 
running a verifier and removing the counterexamples it finds.
This yields the counterexample-guided repair algorithm, 
that was first introduced by~\citet{GoldbergerKatzAdiEtAl2020}, in a similar form.
Removing counterexamples corresponds to a scenario optimisation problem~\(\cxremove_{N}\) 
of the robust optimisation problem~\(R\) from Equation~\eqref{eqn:r}, where 
\begin{equation}
  \cxremove_{N} : \left\{\!
    \arraycolsep=2pt
    \begin{array}{cl}
      \underset{\varvec{\theta} \in \Reals^p}{\text{minimise}} & J(\varvec{\theta}) \\
      \text{subject to} & \fSAT{\left(\NN{\left(\vec{x}^{(i)}\right)}\right)} \geq 0 \\
                        & \qquad\qquad\quad \forall i \in \{1, \ldots, N\}.
    \end{array}
  \right.\label{eqn:cr}
\end{equation}
Algorithm~\ref{algo:cgr} defines the counterexample-guided repair algorithm 
using~\(\cxremove_{N}\) and~\(V\) from Equation~\eqref{eqn:v}.
In analogy to~\(\cxremove_N\), we use~\(V_N\) to denote the verification problem 
in iteration~\(N\).
Concretely,
\begin{equation}
  V_N : \left\{\!
    \arraycolsep=2pt
    \begin{array}{cl}
      \underset{\vec{x}}{\text{minimise}} & \fSAT(\NN[\theta^{(N-1)}](\vec{x})) \\
      \text{subject to} & \vec{x} \in \mathcal{X_\varphi}.
    \end{array}
  \right.\label{eqn:vN}
\end{equation}
We call the iterations of Algorithm~\ref{algo:cgr} \emph{repair steps}.

Algorithm~\ref{algo:cgr} is concerned with repairing a single property.
However, the algorithm extends to repairing multiple properties by adding one 
constraint for each property to~\(\cxremove_N\) and verifying the properties separately.
While Algorithm~\ref{algo:cgr} is formulated for the repair problem~\(R\),
it is easy to generalise it to any robust program~\(P\) as defined in Equation~\eqref{eqn:p}.
Then, solving~\(\cxremove_N\) corresponds to solving~\(\scenprob\) from Equation~\eqref{eqn:sp}
and solving~\(V_N\) corresponds to finding maximal constraint violations of~\(P\).

The question we are concerned with in this paper is whether~Algorithm~\ref{algo:cgr} is 
guaranteed to terminate after finitely many repair steps.
We investigate this question in the following section by studying robust programs 
that are similar to the repair problem for neural networks and 
typical specifications, but either more restrained or more general.

\begin{algorithm}[tb]
  \caption{Counterexample-Guided Repair}\label{algo:cgr}
  \begin{algorithmic}
  \STATE \(N \leftarrow 0\)
  \REPEAT
    \STATE \COMMENT{counterexample-removal~\eqref{eqn:cr}}
    \STATE \(\varvec{\theta}^{(N)} \leftarrow\)~local minimiser of~\(\cxremove_N\)
    \STATE \(N \leftarrow N + 1\)\;
    \STATE \COMMENT{verification~\eqref{eqn:vN}}
    \STATE \(\vec{x}^{(N)} \leftarrow\)~global minimiser of~\(V_N\)
  \UNTIL{\(\fSAT{\left(\NN[\varvec{\theta}^{(N-1)}]{\left(\vec{x}^{(N)}\right)}\right)} \geq 0\)}
  \end{algorithmic}
\end{algorithm}

\section[Termination Results for Algorithm 1]{Termination Results for Algorithm~\ref{algo:cgr}}\label{sec:theory}
The counterexample-guided repair algorithm (Algorithm~\ref{algo:cgr}) repairs neural networks
by iteratively searching and removing counterexamples. 
In this section, we study whether Algorithm~\ref{algo:cgr} is guaranteed to terminate
and whether it produces optimally repaired networks.
Our primary focus is on studying robust optimisation problems that are more
restrained or more general than the repair problem~\(R\) from Equation~\eqref{eqn:r}.
We apply Algorithm~\ref{algo:cgr} to such problems and study termination of the 
algorithm for these problems.
On our way, we also address the related questions of optimality on termination and
termination when using an early-exit verifier as introduced in Definition~\ref{dfn:optimal-early-exit-verifiers}.
 
\begin{table}[tb]
  \caption[Symbol Overview]{Symbol Overview}\label{tab:symbol-overview}
  \centering
  \begin{tabular}{ll}
    \textbf{Symbol} & \multicolumn{1}{c}{\textbf{Meaning}} \\\midrule
    \(R\) & Repair Problem~\eqref{eqn:r} \\
    \(\cxremove_N\) & Counterexample Removal Problem~\eqref{eqn:cr} \\
    \(V_N\) & Verification Problem for~\(\NN[\varvec{\theta}^{(N-1)}]\)~\eqref{eqn:vN} \\
    \(\varvec{\theta}^\dagger\) & (Local) minimiser of \(R\) \\
    \(\varvec{\theta}^{(N)}\) & (Local) minimiser of \(\cxremove_N\) \\
    \(\vec{x}^{(N)}\) & Global minimiser of \(V_N\)
  \end{tabular}
\end{table}
Table~\ref{tab:symbol-overview} summarises the central problem and variable names
that we use throughout this paper.
The iterations of Algorithm~\ref{algo:cgr} are called \emph{repair steps}.
We count the repair steps starting from one but index the 
counterexample-removal problems starting from zero, reflecting the number of constraints.
Hence, the minimiser of the counterexample-removal problem~\(\cxremove_{N-1}\)
from Equation~\eqref{eqn:cr} in repair step~\(N\) 
is~\(\varvec{\theta}^{(N-1)}\).
The verification problem in repair step~\(N\) is~\(V_N\) with a 
global minimiser~\(\vec{x}^{(N)}\).
%
%
%
\subsection{Optimality on Termination}
We prove that when applied to any robust program~\(P\) as defined in Equation~\eqref{eqn:p},
counterexample-guided repair produces a minimiser of~\(P\) 
whenever it terminates.
While Algorithm~\ref{algo:cgr} is formulated for the repair problem~\(R\),
it is easy to generalise it to~\(P\), as described in Section~\ref{sec:cgr}.
The following proposition holds regardless of whether we search for a 
local minimiser or a global minimiser of~\(R\). 

\begin{proposition}[Optimality on Termination]\label{lem:cgr-optimality}
  Whenever Algorithm~\ref{algo:cgr} terminates after~\(\overline{N}\) 
  iterations, it holds 
  that~\(\varvec{\theta}^{(\overline{N}-1)} = \varvec{\theta}^\dagger\).
\end{proposition}

We defer the proof of Proposition~\ref{lem:cgr-optimality} 
to Appendix~\ref{sec:proofs-optimality}.

\subsection{Non-Termination for General Robust Programs}%
\label{sec:non-term-general}
In this section, we demonstrate non-termination and divergence of
Algorithm~\ref{algo:cgr} when we relax the assumptions 
on the repair problem~\(R\) that we outline 
in Section~\ref{sec:prelim-nn-repair}.
In particular, we drop the assumption that the property's 
input set~\(\mathcal{X}_\varphi\) is bounded.
We disprove termination by example when~\(\mathcal{X}_\varphi\) is unbounded. 
To simplify the proof, we use a non-standard neural network architecture.
We also present a fully-connected neural network (FCNN)
that similarly leads to non-termination.
However, for this FCNN we also have to relax the assumption 
that we repair all parameters of a neural network jointly. 
Instead, we repair an individual parameter of the FCNN in isolation. 

\begin{proposition}[General Non-Termination]\label{prop:cgr-non-term}
  Algorithm~\ref{algo:cgr} is not guaranteed to terminate 
  for~\(J: \Reals^2 \to \Reals\),~\(\fSAT: \Reals^2 \to \Reals\),
  and~\(\NN: \Reals \to \Reals^2\), where~\(%
    J(\varvec{\theta}) = {\NN(0)}_1 
  \),~\(%
    \fSAT(\vec{y}) = \vec{y}_2 + \vec{y}_1 - 1
  \),~\(\mathcal{X}_\varphi = \Reals\), and
  \begin{equation}
    {\NN(\vec{x})} = \ReLU{\begin{pmatrix}
        \varvec{\theta}_1\varvec{\theta}_2 \\ 
        \varvec{\theta}_1\vec{x} + \varvec{\theta}_2
    \end{pmatrix}}\!\!\!\!\!\;,
  \end{equation}
  where~\(\ReLU{x} = \max(0, x)\) denotes the ReLU.\@
\end{proposition}
\begin{proofidea}
  The network in Proposition~\ref{prop:cgr-non-term} is constructed such that
  it allows for an execution of Algorithm~\ref{algo:cgr} 
  where the counterexamples~\(\vec{x}^{(N)}\) 
  and the parameter iterates~\(\varvec{\theta}^{(N)}\)
  diverge, such that Algorithm~\ref{algo:cgr} does not terminate.
  The detailed proof and a visualisation of this phenomenon 
  are contained in Appendix~\ref{sec:proofs-thm-cgr-non-term}.
\end{proofidea}

\begin{example}[Non-Termination for an FCNN]\label{example:non-term-fcnn-example}
  The network in Proposition~\ref{prop:cgr-non-term} fits our definition of a neural network
  but does not have a standard neural network architecture.
  However, Algorithm~\ref{algo:cgr} also does not terminate
  for repairing only 
  the parameter~\(\varvec{\theta}\) of the FCNN
  \begin{equation}
    \NN(\vec{x}) =
    \ReLU{%
      \!\begin{pmatrix}
        -1 & 0 \\ 1 & -1
      \end{pmatrix}
      \ReLU{%
        \!\begin{pmatrix}
          0 \\ 1
        \end{pmatrix}
        \vec{x}
        + \begin{pmatrix}
          \varvec{\theta} \\ 2
        \end{pmatrix}\!
      }\!\!\!\!\!\,
      +\! \begin{pmatrix}
        2 \\ 0
      \end{pmatrix}\!
    }\!\!\!\!\!\;,\label{eqn:non-term-fcnn-example}
  \end{equation}
  when~\(\fSAT\) and~\(J(\varvec{\theta})\) are 
  as in Proposition~\ref{prop:cgr-non-term}.
  The proof of non-termination for this FCNN 
  is discussed in Appendix~\ref{sec:proofs-thm-cgr-non-term-fcnn}.
\end{example}

\subsection{Robust Programs with Linear Constraints}\label{sec:term-for-lin-constrs}
In the previous section, we relax assumptions on neural network repair
and show non-termination for the resulting more general problem.
In this section, we look at a more restricted class of 
problems instead: robust problems with linear constraints.
For this class, we can prove termination regardless of the objective~\(J\).
Therefore, this class encompasses such cases as training
a linear regression model, a linear support vector machine (SVM),
or a deep linear network~\citep{SaxeMcClellandGanguli2014} to conform 
to a linear specification.
Linear specifications, as defined in Section~\ref{sec:nn-verif}, 
only consist of properties with an affine satisfaction function 
and a closed convex polytope as an input set.

\begin{theorem}[Termination for Linear Constraints]\label{thm:lin-term}
  Let~\(g(\varvec{\theta}, \vec{x}) = \fSAT(\NN(\vec{x}))\) 
  be linear in~\(\vec{x}\) 
  and let~\(\mathcal{X}_\varphi\) be a closed convex polytope. 
  Algorithm~\ref{algo:cgr} computes a minimiser of
  \begin{equation}
    R : \left\{\!
      \arraycolsep=2pt
      \begin{array}{cl}
        \underset{\varvec{\theta} \in \Reals^p}{\text{minimise}} & J(\varvec{\theta}) \\
        \text{subject to} & g{\left(\varvec{\theta}, \vec{x}\right)} \geq 0 \quad \forall \vec{x} \in \mathcal{X}_\varphi,
      \end{array}
    \right.\label{eqn:thm-lin-term-r}
  \end{equation}
  in a finite number of repair steps.
\end{theorem}
\begin{proofidea}
  For~\(R\) as in Equation~\eqref{eqn:thm-lin-term-r},
  all verification problems~\(V_N\) are linear programs 
  sharing the same feasible set~\(\mathcal{X}_\varphi\).
  Due to this, all most-violating counterexamples~\(\vec{x}^{(N)}\) 
  are vertices of~\(\mathcal{X}_\varphi\),
  of which there are only finitely many.
  This forces Algorithm~\ref{algo:cgr} to terminate.
  The detailed proof is contained in Appendix~\ref{sec:proofs-thm-lin-term}.
\end{proofidea}
The insights from our proof enable a new repair algorithm 
for linear regression models based on quadratic programming.
We discuss and evaluate this algorithm in Section~\ref{sec:experiments-linear}.

\subsection{Element-Wise Monotone Constraints}
Next, we study a different restricted class of repair problems that contains
repairing single ReLU and sigmoid neurons to conform to linear specifications.
This includes repairing linear classifiers, 
which are single sigmoid neurons.
In this class of problems, the constraint 
function~\(g(\varvec{\theta}, \vec{x}) = \fSAT(\NN(\vec{x}))\)
is \emph{element-wise monotone} and continuous
and~\(\mathcal{X}_\varphi\) is a hyper-rectangle.
We show termination for this class.

Element-wise monotone functions are monotone in each argument,
all other arguments being fixed at some value.
They can be monotonically increasing and decreasing
in the same element but only for different values of the remaining elements.
We formally define element-wise monotonicity in Appendix~\ref{sec:proofs-thm-elem-monot-term}.
The definition includes single ReLU and sigmoid neurons.

\begin{theorem}[Termination for Element-Wise Monotone Constraints]\label{thm:elem-monot-hyperrect-term}
  Let~\(g(\varvec{\theta}, \vec{x}) = \fSAT(\NN(\vec{x}))\) be element-wise monotone and continuous.
  Let~\(\mathcal{X}_\varphi\) be a hyper-rectangle. 
  Algorithm~\ref{algo:cgr} computes a minimiser of
  \begin{equation}
    R : \left\{\!
      \arraycolsep=2pt
      \begin{array}{cl}
        \underset{\varvec{\theta} \in \Reals^p}{\text{minimise}} & J(\varvec{\theta}) \\
        \text{subject to} & g{\left(\varvec{\theta}, \vec{x}\right)} \geq 0 \quad \forall \vec{x} \in \mathcal{X}_\varphi
      \end{array}
    \right.
  \end{equation}
  in a finite number of repair steps under the assumption that the algorithm 
  prefers global minimisers of~\(V_N\) that are vertices 
  of~\(\mathcal{X}_\varphi\).
\end{theorem}
The assumption in this theorem is weak, as we show in 
Appendix~\ref{sec:proofs-thm-elem-monot-term}.
In particular, it is easy to construct a global minimiser of~\(V_N\)
that is a vertex of~\(\mathcal{X}_\varphi\) 
given any global minimiser of~\(V_N\).
Given that all~\(\vec{x}^{(N)}\) are vertices of~\(\mathcal{X}_\varphi\) 
under this assumption, 
Theorem~\ref{thm:elem-monot-hyperrect-term} follows analogously 
to Theorem~\ref{thm:lin-term}. 
Appendix~\ref{sec:proofs-thm-elem-monot-term} contains a detailed proof.
%

\subsection{Neural Network Repair with Bounded Input Sets}
\begin{table*}[tb]
  \caption[Summary of Termination Results for Algorithm 1]{%
    Summary of Termination Results for Algorithm~\ref{algo:cgr}
  }\label{tab:theory-results-summary}%
  \centering%
  \renewcommand{\arraystretch}{1.33}%
  \begin{tabular}{p{4.5cm}p{5cm}p{3cm}c@{\hspace{0.2cm}}l}
    \multicolumn{1}{c}{\textbf{Problem Class}} & 
    \multicolumn{1}{c}{\textbf{Model}} & 
    \multicolumn{1}{c}{\textbf{Specification}} & 
    \multicolumn{2}{l}{\textbf{Termination}} \\\midrule
    \(\fSAT{\left(\NN(\vec{x})\right)}\)~linear in~\(\vec{x}\), \(\mathcal{X}_\varphi\)~closed convex polytope 
    & Linear Regression Model, Linear SVM, Deep Linear Network & Linear & \(\checkmark\) & \mbox{(Theorem~\ref{thm:lin-term})} \\
    \(\fSAT{\left(\NN(\vec{x})\right)}\)~elem.-wise mon.\ and cont., \(\mathcal{X}_\varphi\)~hyper-rectangle 
    & Linear Classifier, ReLU Neuron & Linear & \(\checkmark\) & \mbox{(Theorem~\ref{thm:elem-monot-hyperrect-term})} \\
    \(\mathcal{X}_\varphi\)~bounded & 
    Neural Network & Bounded Input Set & ? & \\
    \(\mathcal{X}_\varphi\)~unbounded & 
    Neural Network & Unbounded Input Set&  
    \(\bigtimes\) & \mbox{(Proposition~\ref{prop:cgr-non-term})} \\
    Using an early-exit verifier & Any & Any & 
    \(\bigtimes\) & \mbox{(Proposition~\ref{prop:early-exit-non-term})}
  \end{tabular}
  \renewcommand{\arraystretch}{1.0}  
\end{table*}
Table~\ref{tab:theory-results-summary} summarises our results regarding
the termination of Algorithm~\ref{algo:cgr}.
On the one hand, Theorem~\ref{thm:elem-monot-hyperrect-term} provides us 
with a termination guarantee for repairing single neurons.
On the other hand, Proposition~\ref{prop:cgr-non-term}
shows that Algorithm~\ref{algo:cgr} is not, in general, guaranteed
to terminate when applied to neural networks.
However, both results do not readily transfer to our 
primary target~---~neural network repair with 
bounded input sets.
When looking at neural network repair, the verification problem can 
have a minimiser anywhere inside the feasible region.
Furthermore, this minimiser may move when the network parameters are modified.
Therefore, the reasoning we use for proving Theorems~\ref{thm:lin-term}
and~\ref{thm:elem-monot-hyperrect-term} is not directly applicable
when repairing neural networks.
Coming from the other side, Proposition~\ref{prop:cgr-non-term} relies on
constructing a diverging sequence of counterexamples. 
However, when counterexamples need to lie in a bounded set, 
as it is the case with common neural network
specifications, it becomes intricate to construct a diverging 
sequence originating from a repair problem.

In summary, although we can not answer at this point whether 
Algorithm~\ref{algo:cgr} terminates when applied
to neural network repair for bounded property input sets, 
our methodology is useful for studying related questions.
In the following section, we continue our theoretical analysis, 
showing that early-exit verifiers
are insufficient for guaranteeing termination 
of Algorithm~\ref{algo:cgr}.

\begin{figure*}[tb]
  \centering
  \begin{subfigure}{0.5\textwidth}
    \centering
    \includegraphics[width=0.85\textwidth]{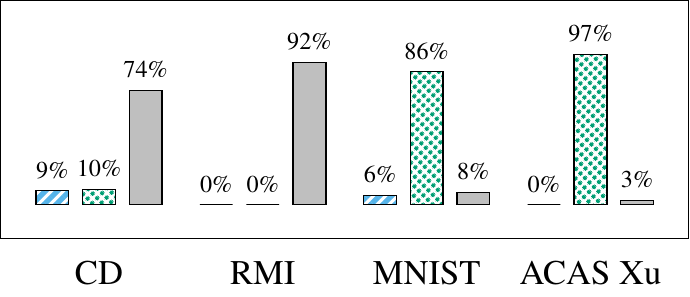}
    \caption{Which is faster in terms of runtime?}%
    \label{fig:optimal-vs-early-exit-faster-runtime}
  \end{subfigure}%
  \begin{subfigure}{0.5\textwidth}
    \centering
    \includegraphics[width=0.85\textwidth]{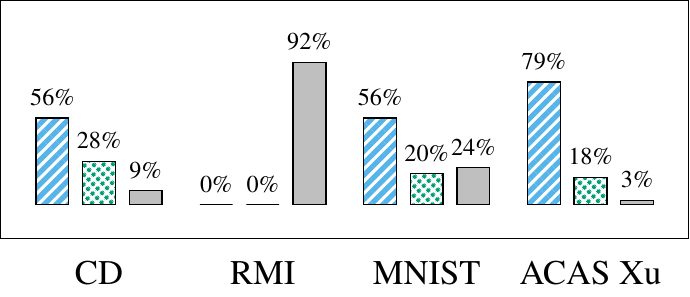}
    \caption{Which is faster in terms of repair steps?}%
    \label{fig:optimal-vs-early-exit-faster-repair-steps}
  \end{subfigure}
  \caption[Optimal vs.\ Early-Exit Verifer]{%
    Optimal vs.\ Early-Exit Verifier.
    We plot how frequently repair using the optimal 
    verifier~\includegraphics{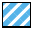}
    or the early-exit verifier~\includegraphics{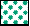}
    is faster in terms of~(\subref{fig:optimal-vs-early-exit-faster-runtime})
    runtime and~(\subref{fig:optimal-vs-early-exit-faster-repair-steps})
    repair steps.
    Grey bars~\includegraphics{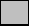}
    depict how frequently both approaches are equally fast. 
    We consider two runtimes equal when they deviate by at most~\num{30} seconds.
    We use four different datasets: CollisionDetection (CD), integer datasets (RMI),
    MNIST and ACAS Xu.
    Gaps to \SI{100}{\percent} are due to failing repairs.
  }\label{fig:optimal-vs-early-exit-faster}
\end{figure*}
\subsection{Early-Exit Verifiers}%
\label{sec:theory-early-exit-verifiers}
From a verification perspective, verifiers are not required to find most-violating
counterexamples.
Instead, it suffices to find any counterexample if one exists.
In this section, we show that using just any counterexample is not sufficient for 
Algorithm~\ref{algo:cgr} to terminate, even for linear regression models.
%
%
Consider a modification of Algorithm~\ref{algo:cgr}, where we only search for a
feasible point of~\(V_N\) with a negative objective value instead of the global
minimum.
This corresponds to using an \emph{early-exit} verifier during repair. 
The following proposition demonstrates that this modification can lead 
to non-termination.

\begin{proposition}[Non-Termination for Early-Exit Verifiers]%
  \label{prop:early-exit-non-term}
  Algorithm~\ref{algo:cgr} modified to use an early-exit verifier is 
  not guaranteed to terminate for~\(%
    J, \fSAT, \NN: \Reals \to \Reals
  \), where~\(%
    J(\varvec{\theta})  = |\varvec{\theta}|
  \),~\(%
    \fSAT(\vec{y})      = \vec{y}
  \),~\(\NN(\vec{x}) = \varvec{\theta} - \vec{x}\), 
  and~\(%
    \mathcal{X}_\varphi = [0, 1]
  \).
\end{proposition}
\begin{proofidea}
  Assume that the early-exit verifier generates the 
  sequence~\(\textstyle\vec{x}^{(N)} = \frac{1}{2} - \frac{1}{N+2}\).
  This leads to non-termination of Algorithm~\ref{algo:cgr}.
  The detailed proof of Proposition~\ref{prop:early-exit-non-term} 
  is contained in Appendix~\ref{sec:proofs-early-exit-non-term}.
\end{proofidea}

This result concludes our theoretical investigation.
In the following section, we research empirical aspects of Algorithm~\ref{algo:cgr},
including the practical implications of the above result on using early-exit verifiers
during repair.

\section{Experiments}\label{sec:experiments}
\emph{Optimal} verifiers that compute most-violating counterexamples are theoretically
advantageous but not widely available~\citep{StrongWuZeljicEtAl2021}.
Conversely, \emph{early-exit} verifiers that produce plain counterexamples without further qualifications
are readily available~\citep{%
  KatzHuangIbelingEtAl2019,%
  BakTranHobbsEtAl2020,%
  TranYangLopezEtAl2020,%
  ZhangWangXuEtAl2022,%
  FerrariMuellerJovanovicEtAl2022%
}, but are theoretically disadvantageous, as apparent 
from Section~\ref{sec:theory-early-exit-verifiers}.
In this section, we empirically compare the effects of using most-violating 
counterexamples and sub-optimal counterexamples~---~as produced by early-exit verifiers and 
falsifiers~---~for repair.
Additionally, we apply our insights from Section~\ref{sec:term-for-lin-constrs} for repairing
linear regression models.
Appendix~\ref{sec:experiments-extra} contains additional experimental results.
Our experiments address the following questions regarding counterexample-guided repair:
\begin{enumerate}
  \item How does repair using an early-exit verifier compare quantitatively to repair using an optimal verifier?
  \item What quantitative advantages does it provide to use falsifiers during repair?
  \item Can we surpass existing repair algorithms for linear regression models using 
  our theoretical insights?
\end{enumerate}

In our experiments, we repair 
an MNIST~\citep{LeCunBottouBengioEtAl1998} image classification network,
ACAS Xu aircraft control networks~\citep{JulianKochenderferOwen2018,KatzBarrettDillEtAl2017}, 
a CollisionDetection~\citep{Ehlers2017} particle dynamics network, 
and integer dataset Recursive Model Indices (RMIs)~\citep{TanZhuGuo2021}
for database indexing.
RMIs contain two stages: a first-stage neural network
and several second-stage linear regression models.
We collect~\num{50} repair instances for MNIST,~\num{34} for 
ACAS Xu,~\num{100} for CollisionDetection,~\num{50} for RMI first-stage
networks and~\num{100} for RMI second-stage models.
A detailed description of all datasets, networks and specifications is
contained in Appendix~\ref{sec:experiment-design}.

For repair, we make use of an early-exit verifier, an optimal verifier,
and the BIM falsifier~\citep{KurakinGoodfellowBengio2016,MadryMakelovSchmidtEtAl2018}.
To obtain an optimal verifier,
we modify the ERAN verifier~\citep{SinghMuellerBalunovicEtAl} 
to compute most-violating counterexamples. 
This is described in Appendix~\ref{sec:experiments-design-eran}.
We use the modified ERAN verifier both as the early-exit and 
as the optimal verifier. 
We use the SpecRepair counterexample-removal 
algorithm~\citep{BauerMarquartBoetiusLeueEtAl2021} unless otherwise noted. 
Our implementation and hardware are documented in 
Appendix~\ref{sec:experiment-design-impl-hardware}.
Our source code is available at~\url{https://github.com/sen-uni-kn/specrepair}.
 
\subsection{Optimal vs. Early Exit Verifier}\label{sec:experiments-optimal-vs-early-exit}
To evaluate how repair using an early-exit verifier compares to repair using an 
optimal verifier, we run repair using both verifiers for CollisionDetection, MNIST,
ACAS Xu, 
and the RMI first-stage networks.
%

Figure~\ref{fig:optimal-vs-early-exit-faster} depicts which verifier
leads to repair fastest.
The figure shows this both for the absolute runtime of repair and 
the number of repair steps.
For the larger MNIST and ACAS Xu networks, we observe that repair using the
early-exit verifier requires less runtime in most cases. 
Regarding the number of repair steps, we observe the opposite trend. 
Here, the optimal verifier yields repair in fewer repair steps more often than not.
The additional runtime cost of computing
most-violating counterexamples offsets the advantage in repair steps.
%
%
For the smaller CollisionDetection network and the RMI first stage networks, 
we primarily observe that only infrequently repair using one verifier 
outperforms using the other by more than~\num{30} seconds.
While there is no variation regarding the number of repair steps for the 
Integer Dataset RMIs, Figure~\ref{fig:optimal-vs-early-exit-faster-repair-steps} shows
the same trend for CollisionDetection as for ACAS Xu and MNIST.\@

\subsection{Using Falsifiers for Repair}\label{sec:experiments-falsifiers}
\begin{figure}[tb]
  \centering
  \includegraphics[width=\linewidth]{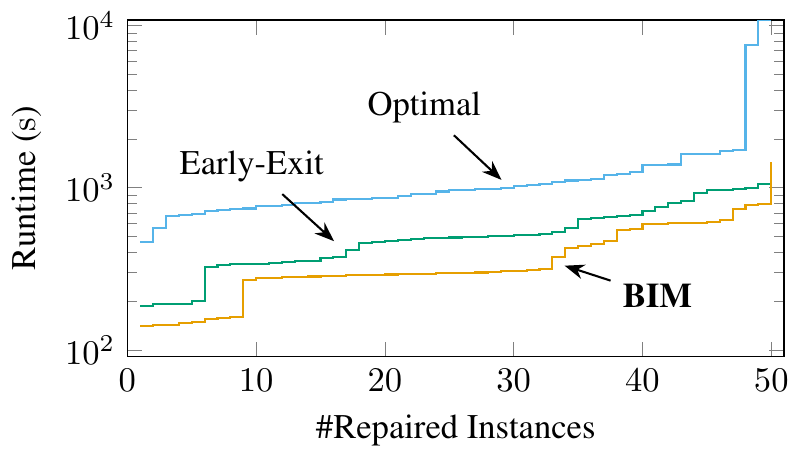}
  %
  \caption[Repair using Falsifiers]{%
    Repair using Falsifiers.
    We plot the number of repaired MNIST instances that individually
    require less than a certain runtime.
    We plot this for repair using
    BIM~\includegraphics{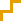},
    only the optimal verifier~\includegraphics{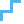}
    and only the early-exit
    verifier~\includegraphics{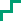}.
  }\label{fig:falsifiers-cactus-mnist}
\end{figure}
Falsifiers are sound but incomplete counterexample searchers that specialise in finding
violations fast. 
In this section, we study how falsifiers can speed up repair. 
For this purpose, we repair an MNIST network using the
BIM~\citep{KurakinGoodfellowBengio2016,MadryMakelovSchmidtEtAl2018} falsifier
that we describe in Section~\ref{sec:nn-verif}.
We start repair by searching counterexamples using BIM.\@
Only when BIM fails to produce further counterexamples we
turn to the early-exit verifier. 
Ideally, we would want that the verifier is invoked only once to prove 
specification satisfaction.
Practically, often several additional repair steps have to be performed using the verifier. 

Figure~\ref{fig:falsifiers-cactus-mnist} summarises the results of our experiment.
We see that using BIM can significantly accelerate repair of the MNIST network,
demonstrating the potential of falsifiers for repair.
BIM is an order of magnitude faster than the early-exit verifier, yet it can find
counterexamples with a larger violation.
Thus, BIM can sometimes provide the repair step advantage of the optimal verifier
at a much smaller cost. 
Appendix~\ref{sec:experiments-extra-falsifiers} contains further experiments
on using falsifiers for repair.
%
%
%

\subsection{Repairing Linear Regression Models}\label{sec:experiments-linear}
\begin{table}[tb]
  \caption[Repairing Linear Regression Models]{%
    Repairing Linear Regression Models.
    We report the success rates of repairing RMI second-stage linear regression models
    for two specifications with different error bounds~\(\varepsilon\).
    The success rates include models that already satisfy their specification.
  }\label{tab:rmi-second-stage-results}
  \centering
  \begin{tabular}{lcccc}
                          & \multicolumn{2}{c}{\textbf{Success Rate}} \\
    \textbf{Algorithm} & \(\varepsilon = 100\) & \(\varepsilon = 150\) \\
    \midrule
    Ouroboros~\citetalias{TanZhuGuo2021}
                          & \SI{30}{\percent} & \SI{77}{\percent}  \\
    SpecRepair~\citetalias{BauerMarquartBoetiusLeueEtAl2021}
                          & \SI{58}{\percent} & \SI{94}{\percent}  \\
  Quadratic Programming & \(\mathbf{72}\,\textbf{\%}\) & \(\mathbf{97}\,\textbf{\%}\)
  \end{tabular} \\[0.5em]
  {\small
    \citetalias{TanZhuGuo2021} \citet{TanZhuGuo2021}\ \ 
    \citetalias{BauerMarquartBoetiusLeueEtAl2021} \citet{BauerMarquartBoetiusLeueEtAl2021}
  }
\end{table}

Our theoretical investigation into the repair of linear regression models in 
Section~\ref{sec:term-for-lin-constrs} provides interesting insights that can be 
used to create a repair algorithm for linear regression models 
based on quadratic programming.
In this section, we describe this algorithm and compare 
it to the Ouroboros~\citep{TanZhuGuo2021} and 
SpecRepair~\citep{BauerMarquartBoetiusLeueEtAl2021} repair algorithms.
Both Ouroboros and SpecRepair are counterexample-guided repair algorithms.
The linear regression models we repair are the second-stage models of
several integer dataset RMIs.
%

\textbf{Insights into Repairing Linear Regression Models}\quad
We recall from Section~\ref{sec:term-for-lin-constrs}
that for repairing a linear regression model 
to conform to a linear specification, the most-violating counterexample 
for a property is always located at one of the vertices of the 
property's input set.
This implies two conclusions for repairing the second-stage RMI models:
\begin{enumerate}
  \item[a)] To verify a linear regression model, it suffices to evaluate it on 
    the vertices of the input set.
    This provides us with an analytical solution of the verification problem~\(V\).
    As the input of the RMI second-stage models is one-dimensional, 
    we only have to check property violation for two points per property.
  \item[b)] Since we can analytically solve~\(V\), we can
    rewrite~\(R'\) from Equation~\eqref{eqn:r-alt} using, 
    for RMI second-stage models, two constraints per property.
    The two constraints correspond to evaluating the satisfaction function
    for the two vertices of the property input set.
    We obtain an equivalent formulation of the repair problem~\(R\) 
    from Equation~\eqref{eqn:r} with a finite number of constraints. 
\end{enumerate}

\textbf{Repair using Quadratic Programming}\quad
Conclusion b) provides an equivalent formulation of
the repair problem~\(R\) with finitely many linear constraints.
We train and repair the second-stage models using MSE.\@
Since MSE is a convex quadratic function and all constraints are linear, 
it follows that the repair problem is a 
quadratic program~\citep{BoydVandenberghe2014}.
This allows for applying a quadratic programming solver to repair the 
linear regression models directly.
We use Gurobi~\citep{GurobiOptimization2021} and report the results for 
repairing linear regression models using this method
under the name~\emph{Quadratic Programming}.
 
Table~\ref{tab:rmi-second-stage-results} summarises the results
of repairing the second-stage RMI linear regression models.
Our new Quadratic Programming repair algorithm achieves the highest success rate,
outperforming Ouroboros and SpecRepair.
In fact, due to solving the repair problem directly, Quadratic Programming is
guaranteed to produce the optimal repaired model whenever repair is possible.
Our implementation of the different algorithms does not allow for a fair runtime comparison, but
we remark that the runtime of Quadratic Programming is competitive in our experiments.

\section{Conclusion}
In this paper, we prove termination of counterexample-guided repair
for linear regression models, linear classifiers and single ReLU neurons,
assuming linear specifications.
We disprove termination for repairing neural networks when the specification
has an unbounded input set.
As our results show, our methodology of viewing repair as robust optimisation is 
useful for studying the theoretical properties of counterexample-guided repair.
Empirically, we find that both early-exit verifiers and falsifiers allow
achieving repair and can give speed advantages.
For repairing linear regression models, we surpass existing approaches by designing
a novel repair algorithm using our theoretical insights.
Overall, we believe that robust optimisation provides a rich arsenal of useful tools 
for studying and advancing repair, both theoretically and practically.

\textbf{Future Work}\quad
Theorems~\ref{thm:lin-term} and~\ref{thm:elem-monot-hyperrect-term}
provide sufficient conditions for termination of Algorithm~\ref{algo:cgr}.
Deriving sufficient conditions that are closer to neural network repair
is an interesting direction for future work, as is deriving necessary
conditions for termination.
Another direction for future work is studying different classes of
verifiers beyond optimal and early-exit verifiers.
Further sufficient conditions, necessary conditions, 
and a more refined taxonomy of verifiers 
could provide insights into why, practically, 
Algorithm~\ref{algo:cgr} terminates even when using early-exit verifiers.

\def\UrlBreaks{\do\/\do-}  
\bibliography{main.bib}
\bibliographystyle{icml2023}

\newpage
\appendix
\onecolumn

\section{Proofs}\label{sec:proofs}
This section contains the full proofs of all our propositions and theorems.

\subsection{Proposition~\ref{lem:cgr-optimality}}\label{sec:proofs-optimality}
\begin{proof}[Proof of Proposition~\ref{lem:cgr-optimality}]
  Assume Algorithm~\ref{algo:cgr} has terminated 
  after~\(\overline{N}\) iterations for some robust program~\(R\). 
  Since Algorithm~\ref{algo:cgr} has terminated, we know 
  that~\(\min V_{\overline{N}} \geq 0\). 
  Hence,~\(\varvec{\theta}^{(\overline{N}-1)}\) is feasible for~\(R\). 
  As~\(\varvec{\theta}^{(\overline{N}-1)}\) also minimises~\(\cxremove_{\overline{N}-1}\), 
  which is a relaxation of~\(R\),
  it follows that~\(\varvec{\theta}^{(\overline{N}-1)}\) minimises~\(R\). 
\end{proof}

This proof is independent of whether we search for a local minimiser or a global
minimiser of~\(R\). 
Therefore, Proposition~\ref{lem:cgr-optimality} holds regardless of the type of minimiser 
of~\(R\) that we are interested in.

\subsection{Proposition~\ref{prop:cgr-non-term}}\label{sec:proofs-thm-cgr-non-term}
For proving Proposition~\ref{prop:cgr-non-term}, we first prove non-termination for
a simplified version of the network in Proposition~\ref{prop:cgr-non-term}.
This simplified version serves as a lemma for proving Proposition~\ref{prop:cgr-non-term}.
\begin{lemma}\label{lem:cgr-non-term-1}
  Algorithm~\ref{algo:cgr} does not terminate 
  for~\(J: \Reals \to \Reals\),~\(\fSAT: \Reals^2 \to \Reals\)
  and~\(\NN: \Reals \to \Reals^2\) where~\(%
    J(\varvec{\theta}) = \NN(0)_1 
  \),~\(%
    \fSAT(\vec{y}) = \vec{y}_2 + \vec{y}_1 - 1
  \),~\(\mathcal{X}_\varphi = \Reals\), and
  \begin{equation}
    {\NN(\vec{x})} = \ReLU{\begin{pmatrix}
      -\varvec{\theta} \\ 
      \varvec{\theta} - \vec{x} 
    \end{pmatrix}}\!\!\!\!\!\;,
  \end{equation}
  where~\(\ReLU{x} = \max(0, x)\) denotes the ReLU.\@
\end{lemma}

\begin{figure*}[tb]
  \centering
  \begin{subfigure}[b]{0.45\textwidth}
    \includegraphics[width=\textwidth]{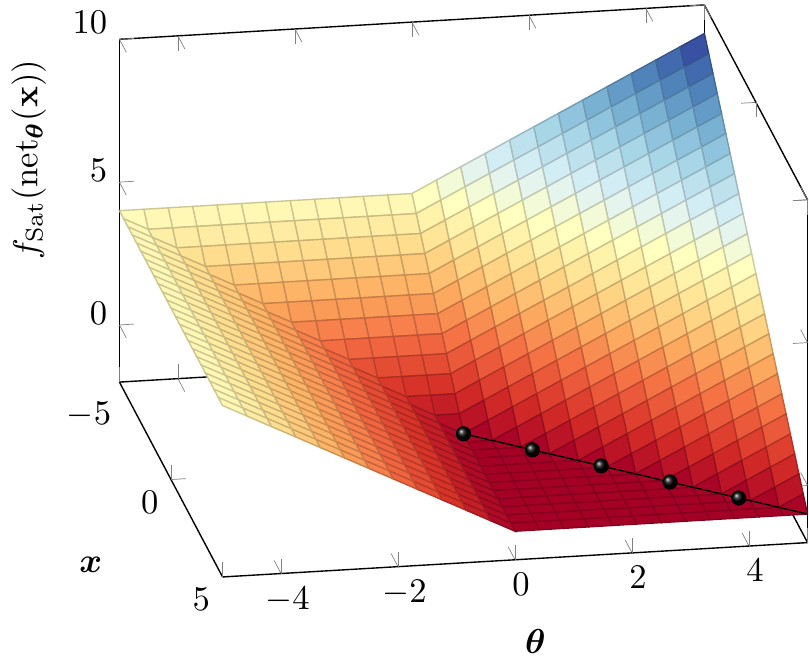}
    \caption{%
      Setting of Proposition~\ref{prop:cgr-non-term}
    }\label{fig:non-term-proof-constraint-prop}
  \end{subfigure}
  \hspace{1.5cm}
  \begin{subfigure}[b]{0.45\textwidth}
    \centering
    \includegraphics[width=\textwidth]{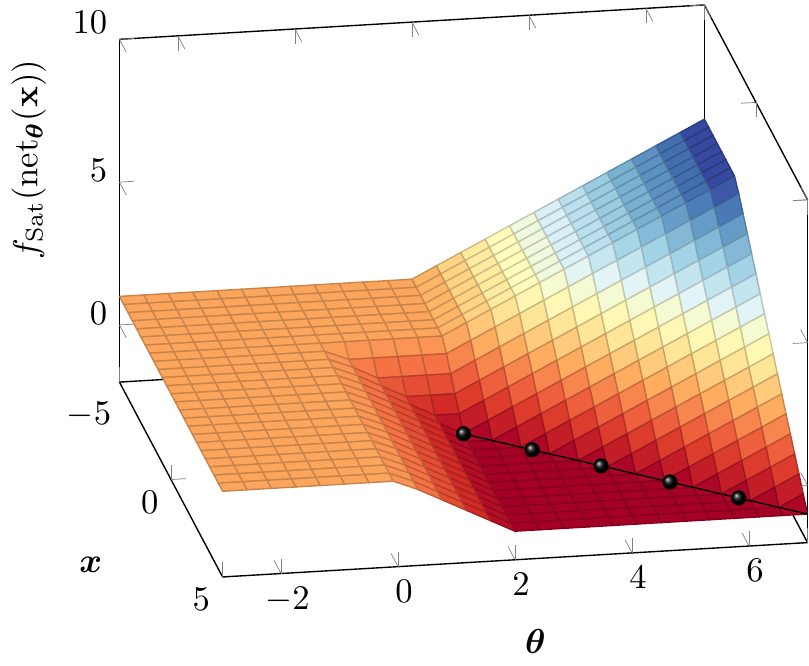}
    \caption{%
      FCNN Variant from Example~\ref{example:non-term-fcnn-example}
    }\label{fig:non-term-proof-constraint-fcnn}
  \end{subfigure}
  \caption[Constraint Visualisations for Non-Termination Proofs]{%
    Constraint Visualisations for Non-Termination Proofs.
    We visualise the function~\(\fSAT(\NN(\vec{x}))\) from Proposition~\ref{prop:cgr-non-term} 
    and for the FCNN variant from Example~\ref{example:non-term-fcnn-example}.
    In both cases, the parameter iterates~\(\varvec{\theta}^{(N)}\) 
    and the counterexamples~\(\vec{x}^{(N)}\) 
    diverge to~\(\infty\) along the dark-red flat surface where 
    the~\(\fSAT\) value is negative.
    This divergence implies non-termination of Algorithm~\ref{algo:cgr}.
    The black line~\includegraphics[scale=1]{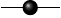}
    represents an example sequence of diverging 
    parameter and counterexample iterates.
  }\label{fig:non-term-proof-constraint}
\end{figure*}

The network~\(\NN\) in Lemma~\ref{lem:cgr-non-term-1} corresponds to the network from
Proposition~\ref{prop:cgr-non-term} with~\(\varvec{\theta}_1 = -1\).
This leads to a one-dimensional input and a one-dimensional parameter space.
Because of this, we can visualise the optimisation landscape that
underlies repairing~\(\NN\).
This visualisation is insightful for the proof of non-termination.
Therefore, before we begin the proof of Lemma~\ref{lem:cgr-non-term-1},
we first give an intuition for the proof using
Figure~\ref{fig:non-term-proof-constraint-prop}.
The core of the proof is that Algorithm~\ref{algo:cgr} generates
parameter iterates~\(\varvec{\theta}^{(N)}\) and counterexamples~\(\vec{x}^{(N)}\) that 
lie on the dark-red flat surface of
Figure~\ref{fig:non-term-proof-constraint-prop}, where~\(\fSAT\) is negative.
The combination of~\(\fSAT\) and the objective function~\(J\) that prefers
non-negative~\(\varvec{\theta}^{(N)}\) leads 
to~\(\varvec{\theta}^{(N)} \geq 0\) for every~\(N \in \Nats\).
As there is always a new counterexample~\(\vec{x}^{(N)}\) 
for every~\(\varvec{\theta}^{(N-1)} \geq 0\),
Algorithm~\ref{algo:cgr} does not
terminate.

\begin{proof}[Proof of Lemma~\ref{lem:cgr-non-term-1}]
  Let~\(J\),~\(\fSAT\),~\(\NN\) and~\(\mathcal{X}_\varphi\) be as in Lemma~\ref{lem:cgr-non-term-1}.
  Assembled into a repair problem, they yield
  \begin{equation}
    R : \left\{\!
      \arraycolsep=2pt
      \begin{array}{cl}
        \underset{\varvec{\theta} \in \Reals}{\text{minimise}} & \ReLU{-\varvec{\theta}} \\
        \text{subject to} & \ReLU{\varvec{\theta} - \vec{x}} + \ReLU{-\varvec{\theta}} -1 \geq 0 \quad \forall \vec{x} \in \Reals.
      \end{array}
    \right.\label{eqn:non-term-lem-r}
  \end{equation}
  We now show that Algorithm~\ref{algo:cgr} does not terminate when applied to~\(R\).
  The problem~\(\cxremove_0\) is minimising~\(J(\varvec{\theta}) = \ReLU{-\varvec{\theta}}\) 
  without constraints.
  The minimiser of~\(J\) is not unique, but all minimisers satisfy~\(\varvec{\theta}^{(0)} \geq 0\).
  Let~\(\varvec{\theta}^{(0)} \geq 0\) be such a minimiser. 

  Searching for the global minimiser~\(\vec{x}^{(1)}\) of~\(V_1\), we find that this
  minimiser is non-unique as well. 
  However, all minimisers satisfy~\(\vec{x}^{(1)} \geq \varvec{\theta}^{(0)}\).
  This follows since any minimiser of
  \begin{equation}
    g{\left(\vec{x}, \varvec{\theta}^{(0)}\right)} = 
    \ReLU{\varvec{\theta}^{(0)} - \vec{x}} + \ReLU{-\varvec{\theta}^{(0)}} -1\label{eqn:proof-prop-non-term-1}
  \end{equation}
  minimises~\(\textstyle\ReLU{\varvec{\theta}^{(0)} - \vec{x}}\) as the remaining terms 
  of Equation~\eqref{eqn:proof-prop-non-term-1} are constant regarding~\(\vec{x}\).
  The observation~\(\vec{x}^{(1)} \geq \varvec{\theta}^{(0)}\) applies analogously 
  for later repair steps. 
  Therefore,~\(\vec{x}^{(N)} \geq \varvec{\theta}^{(N-1)}\).

  For any further repair step, we find that all non-negative 
  feasible points~\(\varvec{\theta}\) of~\(\cxremove_N\) satisfy
  \begin{equation}
    \varvec{\theta} \geq \max{\left(\vec{x}^{(1)}, \ldots, \vec{x}^{(N)}\right)} + 1.\label{eqn:non-term-lem-minimiser}
  \end{equation}
  This follows because~\(%
    \textstyle g{\left(\vec{x}^{(i)}, \varvec{\theta}\right)} \geq 0
  \) has to hold for all~\(i \in \{1, \ldots, N\}\)
  for~\(\varvec{\theta}\) to be feasible for~\(\cxremove_N\).
  Now, if~\(\varvec{\theta} \geq 0\), we have 
  \begin{equation}
    g{\left(\vec{x}^{(i)}, \varvec{\theta}\right)}
    = \ReLU{\varvec{\theta} - \vec{x}^{(i)}} + 
      \ReLU{-\varvec{\theta}} -1
    = \ReLU{\varvec{\theta} - \vec{x}^{(i)}} - 1
    \geq 0,\label{eqn:non-term-lem-minimiser-2}
  \end{equation}
  for all~\(i \in \{1, \ldots, N\}\).
  We see that Equation~\eqref{eqn:non-term-lem-minimiser-2} is satisfied
  for all~\(i \in \{1, \ldots, N\}\)
  only if~\(\varvec{\theta}\) is larger than the largest~\(\vec{x}^{(i)}\)
  by at least one. 
  This yields equivalence of Equations~\eqref{eqn:non-term-lem-minimiser-2}
  and~\eqref{eqn:non-term-lem-minimiser}.
  
  As Equation~\eqref{eqn:non-term-lem-minimiser} always has a solution, 
  there always exists a positive feasible point for~\(\cxremove_N\).
  Now, due to~\(J\), any minimiser~\(\varvec{\theta}^{(N)}\) of~\(\cxremove_N\) is positive 
  and hence satisfies Equation~\eqref{eqn:non-term-lem-minimiser}.
  Putting these results together, we obtain
  \begin{subequations}\label{eqn:non-term-lem-params}
    \begin{align}
      \varvec{\theta}^{(0)} &\geq 0 \\
      \vec{x}^{(N)} &\geq \varvec{\theta}^{(N-1)} \\
      \varvec{\theta}^{(N)} &\geq \vec{x}^{(N)} + 1.
    \end{align}
  \end{subequations}
  Inspecting Equation~\eqref{eqn:non-term-lem-r} closely reveals that no
  positive value~\(\varvec{\theta}\) is feasible for~\(R\) as there always
  exists an~\(\vec{x} \geq \varvec{\theta}\).
  However, it follows from Equations~\eqref{eqn:non-term-lem-params}
  that the iterate~\(\varvec{\theta}^{(N)}\) of Algorithm~\ref{algo:cgr}
  is always positive and thus never feasible for~\(R\).
  Since feasibility for~\(R\) is the criterion for Algorithm~\ref{algo:cgr}
  to terminate, it follows that Algorithm~\ref{algo:cgr} does not terminate
  for this repair problem.
\end{proof}

\begin{remark}
  We might be willing to accept non-termination for problems without a minimiser.
  However,~\(R\) from Equation~\eqref{eqn:non-term-lem-r} has a minimiser.
  We have already seen in the proof of Lemma~\ref{lem:cgr-non-term-1} that all
  positive~\(\varvec{\theta}\) are infeasible for~\(R\).
  Similarly, all~\(\varvec{\theta} \in (-1, 0]\) are infeasible.
  However, all~\(\varvec{\theta} \leq -1\) are feasible as
  \begin{equation}
    \ReLU{\varvec{\theta} - \vec{x}} + \ReLU{-\varvec{\theta}} -1 
    \geq \ReLU{-\varvec{\theta}} -1 \geq 0,
  \end{equation}
  for any~\(\vec{x} \in \Reals\).
  For negative~\(\varvec{\theta}\),~\(J\) prefers larger values.
  Because of this, the only minimiser of~\(R\) is~\(\varvec{\theta}^\dagger = -1\).
  Indeed, Algorithm~\ref{algo:cgr} not only fails to terminate but also
  moves further and further away from the optimal solution.
\end{remark}

We now prove Proposition~\ref{prop:cgr-non-term} using Lemma~\ref{lem:cgr-non-term-1}.
As will become clear during the proof, the divergence for Lemma~\ref{lem:cgr-non-term-1}
transfers to Proposition~\ref{prop:cgr-non-term}.

\begin{proof}[Proof of Proposition~\ref{prop:cgr-non-term}]
  Let~\(J\),~\(\fSAT\),~\(\NN\) and~\(\mathcal{X}_\varphi\) be as in Proposition~\ref{prop:cgr-non-term}.
  The repair problem is
  \begin{equation}
    R : \left\{\!
      \arraycolsep=2pt
      \begin{array}{cl}
        \underset{\varvec{\theta} \in \Reals}{\text{minimise}} & \ReLU{\varvec{\theta}_1\varvec{\theta}_2} \\
        \text{subject to} & \ReLU{\varvec{\theta}_1\vec{x} + \varvec{\theta}_2} 
        + \ReLU{\varvec{\theta}_1\varvec{\theta}_2} -1 \geq 0 \quad \forall \vec{x} \in \Reals.
      \end{array}
    \right.\label{eqn:non-term-prop-r}
  \end{equation}
  To show that Algorithm~\ref{algo:cgr} is not guaranteed to terminate for~\(R\), 
  we now construct an execution of Algorithm~\ref{algo:cgr} that does not terminate.
  We first consider~\(\cxremove_0\), which is minimising \(J\) without constraints. 
  Choosing~\(\varvec{\theta}_1 = -1\) and~\(\varvec{\theta}_2 \geq 0\) yields a local minimiser of~\(J\),
  since~\(\textstyle J{\left(\varvec{\theta}\right)} = 0\), 
  which is the global minimum of~\(J\).
  Assuming~\(\varvec{\theta}_2^{(0)} \geq 0\) and~\(\varvec{\theta}_1^{(0)} = -1\), we now show that there
  is an execution of Algorithm 1, such that 
  \begin{equation}
    \forall N \in \mathbb{N}_0: 
      \varvec{\theta}_1^{(N)} = -1
      \wedge 
      \varvec{\theta}_2^{(N)} \geq 0 
      \label{eqn:prop-cgr-non-term-1}
  \end{equation}
  As~\(\varvec{\theta}_1 = -1\) recreates the neural network from Lemma~\ref{lem:cgr-non-term-1}, 
  Proposition~\ref{prop:cgr-non-term} follows from Lemma~\ref{lem:cgr-non-term-1} 
  when Equation~\eqref{eqn:prop-cgr-non-term-1} holds for some execution.
  In the proof of Lemma~\ref{lem:cgr-non-term-1}, we have already 
  shown that there exists a~\(\varvec{\theta}^{(N)}\) 
  satisfying Equation~\eqref{eqn:prop-cgr-non-term-1} that is feasible for~\(\cxremove_N\).
  Since~\(\textstyle J{\left(\varvec{\theta}^{(N)}\right)} = 0\) for any~\(\varvec{\theta}^{(N)}\)
  satisfying Equation~\eqref{eqn:prop-cgr-non-term-1}, there exist such parameters that are a 
  local (in fact global) minimiser of~\(\cxremove_N\).
  Therefore, Algorithm~\ref{algo:cgr} is not guaranteed to terminate 
  when repairing the neural network in Proposition~\ref{prop:cgr-non-term},
  as there exists an execution of Algorithm 1 that does not terminate.
\end{proof}

\begin{remark}
  While the proof of Proposition~\ref{prop:cgr-non-term} constructs
  an execution that does not terminate, 
  the example in Proposition~\ref{prop:cgr-non-term} 
  also permits executions that terminate.
  In the following, we discuss different executions of Algorithm~\ref{algo:cgr}, 
  including the executions that terminate.
  Beyond the execution constructed in the proof of 
  Proposition~\ref{prop:cgr-non-term}, all executions with 
  \begin{equation}
    \forall N \in \mathbb{N}_0: \varvec{\theta}_1 < 0 \wedge \varvec{\theta}_2 \geq 0 
  \end{equation}
  fail to terminate.
  Such executions may, however, converge to a 
  solution~\(\varvec{\theta}_1 = 0\),~\(\varvec{\theta}_2 \geq 1\).
  While still failing to terminate, they do not diverge. 
  Values of~\(\varvec{\theta}_1 > 0\) lead to non-termination, analogously 
  to the case where~\(\varvec{\theta}_1 < 0\).
  Also in this case, Algorithm~\ref{algo:cgr} may converge to a solution with~\(\varvec{\theta}_1 = 0\). 
  There also exist executions where Algorithm~\ref{algo:cgr} terminates, namely when it chooses~\(\varvec{\theta}_1 = 0\)
  at some point during its execution. 
  Choosing~\(\varvec{\theta}_1 = 0\) is a valid choice, as it yields
  local minimisers of~\(J\) and allows removing any set of counterexamples.
  Therefore, for the example in Proposition~\ref{prop:cgr-non-term}, it is possible that Algorithm~\ref{algo:cgr}
  terminates, but there is no guarantee.

  Regarding the plausibility of non-terminating executions, we first remark that it is reasonable 
  to obtain~\(\varvec{\theta}^{(0)}_1 \neq 0\), as neural network training is unlikely to 
  reach~\(\varvec{\theta}^{(0)}_1 = 0\) exactly. 
  Regarding the plausibility of the choice of local minimisers of~\(\cxremove_N\), we consider
  different concrete counterexample-removal algorithms.
  \begin{itemize}
    \item Gradient-based techniques~\citep{%
       PulinaTacchella2010,%
       GoodfellowShlensSzegedy2015,%
       DongSunWangEtAl2020,%
       TanZhuGuo2021,%
       BauerMarquartBoetiusLeueEtAl2021%
     } are unable to remove counterexamples for Proposition~\ref{prop:cgr-non-term}, 
     as~\(\fSAT\) does not provide information on improving
     property violation through its gradient.
     This is because the region where the most violating counterexamples are located is flat. 
     Therefore, these techniques fail to remove counterexamples, which makes it impossible
     to study the termination of Algorithm 1.
   \item For SMT-based techniques~\citep{GoldbergerKatzAdiEtAl2020}, the choice of 
     the local minimum~\(\varvec{\theta}^{(N)}\) depends on the heuristics applied by the SMT solver.
     The SMT solver may choose to increase only one parameter, leading to a non-terminating
     execution, such as the one constructed in the proof of Proposition~\ref{prop:cgr-non-term}.
  \end{itemize}
\end{remark}

\subsubsection{Example~\ref{example:non-term-fcnn-example}}\label{sec:proofs-thm-cgr-non-term-fcnn}
\begin{figure}[tb]
  \centering
  \includegraphics[scale=1.0]{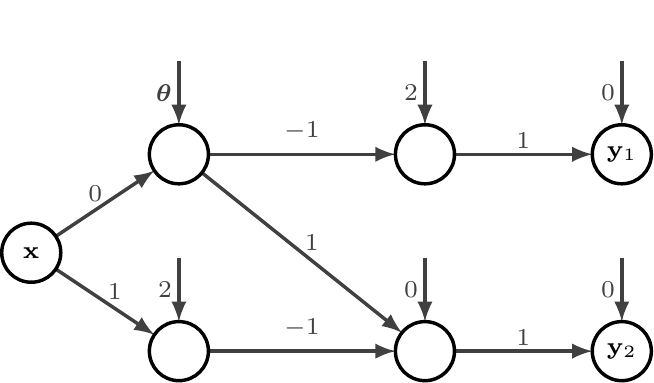}
  \caption[Fully-Connected Neural Network Variant of Proposition~\ref{prop:cgr-non-term}]{%
    Fully-Connected Neural Network Variant of Proposition~\ref{prop:cgr-non-term}.
    This Figure visualises Equation~\eqref{eqn:non-term-fcnn-example}.
    Empty nodes represent single ReLU neurons. 
    Edge labels between nodes contain the network weights. 
    Where edges are omitted, the corresponding weights are zero.
    Biases are written next to the incoming edge above the ReLU neurons.
  }\label{fig:fcnn-non-termination}
\end{figure}

Figure~\ref{fig:fcnn-non-termination} visualises the FCNN
from Example~\ref{example:non-term-fcnn-example}
The proof of non-termination for this FCNN
is analogous to the proof of Lemma~\ref{lem:cgr-non-term-1}.
Figure~\ref{fig:non-term-proof-constraint-fcnn}
visualises~\(\fSAT{\left(\NN(\vec{x})\right)}\) for
the FCNN from Equation~\eqref{eqn:non-term-fcnn-example}.
Comparison with Figure~\ref{fig:non-term-proof-constraint-prop}
reveals that the key aspects of~\(\fSAT{\left(\NN(\vec{x})\right)}\) 
for the FCNN are identical to Lemma~\ref{lem:cgr-non-term-1},
except for being shifted.
Most notably, there also exists a flat surface with a negative~\(\fSAT\) value.
As~\(J\) also prefers non-negative~\(\varvec{\theta}\) in this example,
Algorithm~\ref{algo:cgr} diverges here as well.

\subsection{Theorem~\ref{thm:lin-term}}\label{sec:proofs-thm-lin-term}
\begin{proof}[Proof of Theorem~\ref{thm:lin-term}]
  We prove termination of Algorithm~\ref{algo:cgr} for~\(R\) from Theorem~\ref{thm:lin-term}.
  Optimality then follows from Proposition~\ref{lem:cgr-optimality}.
  Let~\(g: \Reals^p \times \Reals^n \to \Reals\) be linear in the second argument.
  Let~\(\mathcal{X}_\varphi\) be a closed convex polytope.
  Given this, every~\(V_N\) is a linear program and all~\(V_N\) share the same
  feasible set~\(\mathcal{X}_\varphi\).
  Because~\(V_N\) is a linear program, its minimiser coincides with one of the 
  vertices of the feasible set~\(\mathcal{X}_\varphi\). 

  It follows that~\(\forall N \in \Nats: \vec{x}^{(N)} \in \vertices{\mathcal{X}_\varphi}\),
  where~\(\vertices{\mathcal{X}_\varphi}\) are the vertices of~\(\mathcal{X}_\varphi\).
  Because~\(\vertices{\mathcal{X}_\varphi}\) is finite, at some repair step~\(\overline{N}\) 
  of Algorithm~\ref{algo:cgr}, we obtain a minimiser that we already encountered in a 
  previous repair step.
  Let~\(\tilde{N}\) be this previous repair step, such that~\(\vec{x}^{(\tilde{N})} = \vec{x}^{(\overline{N})}\).
  Since~\(\varvec{\theta}^{(\overline{N}-1)}\) is feasible for~\(\cxremove_{\overline{N}-1}\), 
  it satisfies
  \begin{equation}
    0 \leq g{\left(\varvec{\theta}^{(\overline{N}-1)}, \vec{x}^{(\tilde{N})}\right)}
    = g{\left(\varvec{\theta}^{(\overline{N}-1)}, \vec{x}^{(\overline{N})}\right)} 
    = \fSAT{\left(\NN[\varvec{\theta}^{(\overline{N}-1)}]{\left(\vec{x}^{(\overline{N})}\right)}\right)}.
  \end{equation}
  As this is the termination criterion of Algorithm~\ref{algo:cgr}, 
  the algorithm terminates in repair step~\(\overline{N}\).
\end{proof}

\subsection{Theorem~\ref{thm:elem-monot-hyperrect-term}}\label{sec:proofs-thm-elem-monot-term}
We first formally introduce element-wise monotonous functions.
Informally, element-wise monotone functions are monotone in each argument,
all other arguments being fixed at some value.

\begin{defn}[Element-Wise Monotone]\label{dfn:elementwise-monotone}
  A function~\(f: \mathcal{X} \to \Reals\), \(\mathcal{X} \subseteq \Reals^n\), 
  is \emph{element-wise monotone} if
  \begin{equation}
    \forall i \in \{1, \ldots, n\}: 
    \forall \vec{x} \in \mathcal{X}: 
    \restr{f}{\mathcal{X} \cap \left(\{\vec{x}_1\} \times \cdots \times \{\vec{x}_{i-1}\} 
    \times \Reals \times \{\vec{x}_{i+1}\} \times \cdots \times \{\vec{x}_{n}\}\right)}
    \ \text{is monotone}.
  \end{equation}
\end{defn}

\begin{remark}
  Affine transformations of element-wise monotone functions maintain element-wise
  monotonicity. 
  This directly follows from affine transformations maintaining monotonicity.
\end{remark}

Element-wise monotone functions can be monotonically increasing and decreasing
in the same element but only for different values of the remaining elements.
Examples of element-wise monotone functions include the single neurons~\(%
  \textstyle\ReLU{\transp{\vec{w}}\vec{x} + \vec{b}}
\) and~\(%
  \textstyle\sigma{\left(\transp{\vec{w}}\vec{x} + \vec{b}\right)}
\),
where~\(\ReLU{x} = \max(0, x)\) is the ReLU function
and~\(\textstyle \sigma(x) = \frac{1}{1 + e^{-x}}\) is the sigmoid function. 
These functions are also continuous.

In Theorem~\ref{thm:elem-monot-hyperrect-term}, 
we make an assumption on the global minimisers that 
Algorithm~\ref{algo:cgr} prefers when there are multiple global minimisers.
In the proof of Lemma~\ref{lem:elem-monot-minimisers}, 
we show that the assumption in Theorem~\ref{thm:elem-monot-hyperrect-term} 
is a weak assumption.
In particular, we show that it is easy to construct a global minimiser of~\(V_N\)
that is a vertex of~\(\mathcal{X}_\varphi\) given any global minimiser of~\(V_N\).
Lemma~\ref{lem:elem-monot-minimisers} is a preliminary result for proving 
Theorem~\ref{thm:elem-monot-hyperrect-term}.

\begin{lemma}[Optimal Vertices]\label{lem:elem-monot-minimisers}
  Let~\(R\),~\(g\) and~\(\mathcal{X}_\varphi\) be as in Theorem~\ref{thm:elem-monot-hyperrect-term}.
  Then, for every~\(N \in \Nats\) there is~\(\tilde{\vec{x}}^{(N)} \in \vertices{\mathcal{X}_\varphi}\) 
  that globally minimises~\(V_N\),
  where~\(\vertices{\mathcal{X}_\varphi}\) denotes the set of vertices of~\(\mathcal{X}_\varphi\).
\end{lemma}

\begin{proof}
  Let~\(R\),~\(g\),~\(\mathcal{X}_\varphi\) be as in Lemma~\ref{lem:elem-monot-minimisers}. 
  Let~\(N \in \Nats\).
  To prove the lemma we show that a)~\(V_N\) has a minimiser and 
  b) when there is a minimiser of~\(V_N\),
  some vertex of~\(\mathcal{X}_\varphi\) also minimises~\(V_N\) 
  and has the same~\(\fSAT\) value.

  \begin{enumerate}
    \item[a)] As the feasible set of~\(V_N\) is closed and bounded due to being a hyper-rectangle 
      and the objective function is continuous,~\(V_N\) has a minimiser.
    \item[b)] Let~\(\vec{x}^{(N)} \in \Reals^n\) be a global minimiser of~\(V_N\). 
      We show that there is a~\(\tilde{\vec{x}}^{(N)} \in \vertices{\mathcal{X}_\varphi}\)
      such that~\(\tilde{\vec{x}}^{(N)}\) also minimises~\(V_N\) since
      \begin{equation}
        g{\left(\varvec{\theta}^{(N-1)}, \vec{x}^{(N)}\right)} \geq g{\left(\varvec{\theta}^{(N-1)}, \tilde{\vec{x}}^{(N)}\right)}.\label{eqn:lem-elem-monot-minimisers-equal-g}
      \end{equation}

      Pick any dimension~\(i \in \{1, \ldots, n\}\).
      As~\(g\) is element-wise monotone, it is non-increasing in one of the two directions along dimension~\(i\)
      starting from~\(\vec{x}^{(N)}\).
      %
      When~\(\vec{x}^{(N)}\) does not already lie on a face of~\(\mathcal{X}_\varphi\)
      that bounds expansion along the~\(i\)-axis,
      we walk along the non-increasing direction along dimension~\(i\) until we reach 
      such a face of~\(\mathcal{X}_\varphi\).
      As~\(\mathcal{X}_\varphi\) is a hyper-rectangle and, therefore, bounded,
      it is guaranteed that we reach such a face.
      We pick the point on the face of~\(\mathcal{X}_\varphi\) as the new~\(\vec{x}^{(N)}\).
      While keeping dimension~\(i\) fixed, we repeat the above procedure for a different 
      dimension~\(j \in \{1, \ldots, n\}, i \neq j\).
      We iterate the procedure over all dimensions always keeping the value of~\(\vec{x}^{(N)}\)
      in already visited dimensions fixed.

      In every step of this procedure, we restrict ourselves to a lower-dimensional face
      of~\(\mathcal{X}_\varphi\) as we fix the value in one dimension.
      Thus, when we have visited every dimension, we have reached a~\(0\)-dimensional face
      of~\(\mathcal{X}_\varphi\), that is, a vertex.
      Since we only walked along directions in which~\(g\) is non-increasing
      and since~\(g\) is element-wise monotone, the
      vertex~\(\tilde{\vec{x}}^{(N)}\) that we obtain satisfies Equation~\eqref{eqn:lem-elem-monot-minimisers-equal-g}.
      Since~\(\vec{x}^{(N)}\) is a global minimiser, Equation~\eqref{eqn:lem-elem-monot-minimisers-equal-g}
      needs to hold with equality. 
  \end{enumerate}
  Together, a) and b) yield that there is always a vertex~\(\tilde{\vec{x}}^{(N)} \in \vertices{\mathcal{X}_\varphi}\)
  that globally minimises~\(V_N\).
\end{proof}

\begin{proof}[Proof of Theorem~\ref{thm:elem-monot-hyperrect-term}]
  We prove termination with optimality following from Proposition~\ref{lem:cgr-optimality}.
  Let~\(R\),~\(g\),~\(\mathcal{X}_\varphi\) be as in Theorem~\ref{thm:elem-monot-hyperrect-term}. 
  Also, assume that Algorithm~\ref{algo:cgr} prefers vertices of~\(\mathcal{X}_\varphi\)
  as global minimisers of~\(V_N\).
  From Lemma~\ref{lem:elem-monot-minimisers} we know that there is always a vertex of~\(\mathcal{X}_\varphi\)
  that minimises~\(V_N\).
  From the proof of Lemma~\ref{lem:elem-monot-minimisers} we also know that it is easy
  to find such a vertex given any global minimiser of~\(V_N\).
  As Algorithm~\ref{algo:cgr} always chooses vertices of~\(\mathcal{X}_\varphi\) under our 
  assumption, there is only a finite set of minimisers~\(\vec{x}^{(N)}\), 
  as a hyper-rectangle has only finitely many vertices.
  Given this, termination follows analogously to the proof of Theorem~\ref{thm:lin-term}.
\end{proof}

\subsection{Proposition~\ref{prop:early-exit-non-term}}\label{sec:proofs-early-exit-non-term}
\begin{proof}[Proof of Proposition~\ref{prop:early-exit-non-term}]
  Let~\(J\),~\(\fSAT\),~\(\NN\) and~\(\mathcal{X}_\varphi\) be as in Proposition~\ref{prop:early-exit-non-term}.
  When inserting these into Equation~\eqref{eqn:r}, we obtain the repair problem
  \begin{equation}
    R : \left\{\!
      \arraycolsep=2pt
      \begin{array}{cl}
        \underset{\varvec{\theta}\in\Reals}{\text{minimise}} & |\varvec{\theta}| \\
        \text{subject to} & \varvec{\theta} - \vec{x} \geq 0 \quad \forall \vec{x} \in [0, 1].
      \end{array}
    \right.
  \end{equation}
  Assume the early-exit verifier generates the sequence~\(\vec{x}^{(N)} = \frac{1}{2} - \frac{1}{N+2}\)
  as long as these points are counterexamples for~\(\NN[\varvec{\theta}^{(N-1)}]\).
  Otherwise, let it produce~\(\vec{x}^{(N)} = 1\), the global minimum of all~\(V_N\).

  Minimising~\(J\) without constraints yields~\(\varvec{\theta}^{(0)} = 0\).
  The point~\(\vec{x}^{(1)} = \frac{1}{2} - \frac{1}{3}\) is a valid result of the 
  early-exit verifier for~\(V_1\), as it is a counterexample.
  We observe that the constraint
  \begin{equation}
    \fSAT(\NN(\vec{x})) = \varvec{\theta} - \vec{x} \geq 0
  \end{equation}
  is tight when~\(\varvec{\theta} = \vec{x}\). 
  Smaller~\(\varvec{\theta}\) violate the constraint.
  Since~\(J\) prefers values of~\(\varvec{\theta}\) closer to zero, 
  it always holds for any minimiser of~\(\cxremove_N\) that
  \begin{equation}
    \varvec{\theta}^{(N)} = \max{\left(\vec{x}^{(1)}, \ldots, \vec{x}^{(N)}\right)} = \vec{x}^{(N)}.
  \end{equation}
  The last equality is due to the construction of the points returned by the early-exit verifier.
  However, for these values of~\(\varvec{\theta}^{(N)}\),~\(\textstyle \frac{1}{2} - \frac{1}{N+2}\) 
  always remains a valid product of the early-exit verifier for~\(V_N\).
  Thus we obtain,
  \begin{equation}
    \varvec{\theta}^{(N)} = \vec{x}^{(N)} = \frac{1}{2} - \frac{1}{N+2}. 
  \end{equation}
  The minimiser of~\(R\) is~\(\varvec{\theta}^\dagger = 1\).
  However, \(\varvec{\theta}^{(N)}\) does not converge to this point but to the 
  infeasible~\(\lim_{N\to\infty} \varvec{\theta}^{(N)} = \frac{1}{2}\).
  Since the iterates~\(\varvec{\theta}^{(N)}\) always remain infeasible for~\(R\), 
  the modified Algorithm~\ref{algo:cgr} never terminates.
\end{proof}

%

\section{Experiment Design}\label{sec:experiment-design}
In our experiments, we repair 
an MNIST~\citep{LeCunBottouBengioEtAl1998} network,
ACAS Xu networks~\citep{KatzBarrettDillEtAl2017}, 
a CollisionDetection~\citep{Ehlers2017} network, 
and integer dataset Recursive Model Indices (RMIs)~\citep{TanZhuGuo2021}.
For repair, we make use of an early-exit verifier, an optimal verifier,
the SpecAttack falsifier~\citep{BauerMarquartBoetiusLeueEtAl2021},
and the BIM falsifier~\citep{%
  MadryMakelovSchmidtEtAl2018,%
  KurakinGoodfellowBengio2016%
}.
To obtain an optimal verifier,
we modify the ERAN verifier~\citep{SinghMuellerBalunovicEtAl} 
to compute most-violating counterexamples. 
We use the modified ERAN verifier both as the early-exit and 
as the optimal verifier in our experiments, as it supports both exit modes. 

In all experiments, we use the SpecRepair counterexample-removal 
algorithm~\citep{BauerMarquartBoetiusLeueEtAl2021}
unless otherwise noted. 
We set up all verifiers and falsifiers to return a single counterexample.
For SpecAttack, which produces multiple counterexamples, 
we select the counterexample with the largest violation.
We make this modification to eliminate differences due to some tools
returning more counterexamples than others, as we are interested in 
studying the effects of counterexample quality, not counterexample quantity. 

We make our source code available under the 
Apache 2.0 license\footnote{https://www.apache.org/licenses/LICENSE-2.0.html}
at~\url{https://github.com/sen-uni-kn/specrepair}.
Our experimental data is available
at~\url{https://doi.org/10.5281/zenodo.7938547}.

\subsection{Modifying ERAN to Compute Most-Violating Counterexamples}\label{sec:experiments-design-eran}
The ETH Robustness Verifier for Neural Networks (ERAN)~\citep{SinghMuellerBalunovicEtAl} 
combines abstract interpretation with Mixed Integer Linear Programming (MILP)
to verify neural networks. 
For our experiments, we use the DeepPoly abstract interpretation~\citep{SinghGehrPueschelEtAl2019}.
ERAN leverages Gurobi~\citep{GurobiOptimization2021} for MILP.\@
To verify properties with low-dimensional input sets having a large diameter, 
ERAN implements the ReluVal input splitting branch and bound 
procedure~\citep{WangPeiWhitehouseEtAl2018b}.
We employ this branch and bound procedure only for \mbox{ACAS Xu}. 

The Gurobi MILP solver can be configured to stop optimisation when encountering the
first point with a negative satisfaction function value below a small threshold.
We use this feature for the early-exit mode.
To compute most-violating counterexamples, we instead run the MILP
solver until achieving optimality. 

The input-splitting branch and bound procedure evaluates branches in parallel. 
In the early-exit mode, the procedure terminates when it finds a counterexample on any branch. 
As other branches may contain more-violating counterexamples, we search
the entire branch and bound tree in the optimal mode.


\subsection{Datasets, Networks, and Specifications}\label{sec:experiments-design-datasets-etc}
\begin{table}[tb]
  \caption[Network Architectures]{%
    Network Architectures.
    \texttt{In(\textbullet{})} gives the dimension of the network input. 
    Convolutional layers are denoted \texttt{%
      Conv(out=\textbullet{}, kernel=\textbullet{}, stride=\textbullet{}, pad=\textbullet{})%
    }, where \texttt{out} is the number of filters and \texttt{kernel}, \texttt{stride}, 
    and \texttt{pad} are the kernel size, stride, and padding 
    for all spatial dimensions of the layer input. 
    Fully-connected layers are denoted \texttt{FC(out=\textbullet{})}, where \texttt{out} is the number of neurons. 
    [\texttt{\textbullet{}}] \texttimes{} \(n\) denotes the \(n\)-fold repetition of the block 
    in square brackets.
    RMI stands for the Integer Dataset RMIs.
  }\label{tab:network-architectures}
  \centering
  \begin{tabular}{m{0.2\linewidth}m{0.7\linewidth}}
    \textbf{Dataset} & \textbf{Network Architecture} \\
    \midrule
    MNIST & \texttt{In(1\texttimes{}28\texttimes{}28)}, \texttt{Conv(out=8, kernel=3, stride=3, pad=1)}, \texttt{ReLU}, \texttt{FC(out=80)}, \texttt{ReLU}, \texttt{FC(out=10)} \\
    ACAS Xu & \texttt{In(5)}, [\texttt{FC(out=50)}, \texttt{ReLU}] \texttimes{} 6, \texttt{FC(out=5)}\\
    Collision-Detection & \texttt{In(6)}, [\texttt{FC(out=10)}, \texttt{ReLU}] \texttimes{} 2, \texttt{FC(out=2)} \\
    RMI, First Stage & \texttt{In(1)}, [\texttt{FC(out=16)}, \texttt{ReLU}] \texttimes{} 2, \texttt{FC(out=1)} \\
    RMI, Second Stage & \texttt{In(1)}, \texttt{FC(out=1)} \\
  \end{tabular}
\end{table}

We perform experiments with four different datasets.
In this section, we introduce the datasets, as well as 
what networks we repair to conform to which specifications.
The network architectures for each dataset are contained in Table~\ref{tab:network-architectures}.

\subsubsection{MNIST}
The MNIST dataset~\citep{LeCunBottouBengioEtAl1998} consists of~\num{70000} labelled 
images of hand-written Arabic digits.
Each image consists of~\(28 \times 28\) pixels.
The dataset is split into a training set of~\num{60000} images and a test set 
of~\num{10000} images.
The task is to predict the digit in an image from the image pixel data. 
We train a small convolutional neural network achieving~\SI{97}{\percent} test set 
accuracy (\num{98}\% training set accuracy). 
Table~\ref{tab:network-architectures} contains the concrete architecture.

We repair the~\(L_\infty\) adversarial robustness of this convolutional neural
network for groups of~\num{25} input images.
These images are randomly sampled from the images in the training set 
for which the network is not robust. 
Each robustness property has a radius of~\num{0.03}.
Overall, we form~\num{50} non-overlapping groups of input images. 
Thus, each repaired network is guaranteed to be locally robust for
a different group of~\num{25} training set images. 
While specifications of this size are not practically relevant, 
they make it feasible to perform several (\num{50}) experiments
for each verifier variant. 
We formally define \(L_\infty\) adversarial robustness in Appendix~\ref{sec:specs-linf-robustness}.

We train the MNIST network using Stochastic Gradient Descent (SGD)
with a mini-batch size of~\num{32}, a learning rate of~\num{0.01}
and a momentum coefficient of~\num{0.9}, training for two epochs.
Counterexample-removal uses the same setup, except for using a decreased
learning rate of~\num{0.001} and iterating only for a tenth of an epoch. 

\subsubsection{ACAS Xu}
The ACAS Xu networks~\citep{KatzBarrettDillEtAl2017} form a 
collision avoidance system for aircraft without onboard personnel. 
Each network receives five sensor measurements that characterise an encounter
with another aircraft. 
Based on these measurements, an ACAS Xu network computes scores for five possible steering
directions: Clear of Conflict (maintain course), weak left/right, and strong left/right. 
The steering direction advised to the aircraft is the output with
the minimal score. 
Each of the~\num{45} ACAS Xu networks is responsible for another class of encounter scenarios.
More details on the system are provided by~\citet{JulianKochenderferOwen2018}.
Each ACAS Xu network is a fully-connected ReLU network with 
six hidden layers of~\num{50} neurons each.

\citet{KatzBarrettDillEtAl2017}~provide safety specifications for the ACAS Xu
networks. 
Of these specifications, the property~\(\phi_2\) is violated by the
largest number of networks. 
We repair~\(\phi_2\) for all networks violating it, yielding~\num{34}
repair cases. 
The property~\(\phi_2\) specifies that the score for the Clear of Conflict 
action is never maximal (least-advised) when the intruder is far away and
slow. 
The precise formal definition of~\(\phi_2\) is given in Appendix~\ref{sec:specs-acasxu}.

We repair the ACAS Xu networks following~\citet{BauerMarquartBoetiusLeueEtAl2021}.
To replace the unavailable ACAS Xu training data, we randomly
sample a training and a validation set and compare with the scores produced
by the original network.
As a loss function, we use the asymmetric mean square error loss 
of~\citet{JulianKochenderfer2019}.
We repair using the Adam training algorithm~\citep{KingmaBa2015}
with a learning rate of~\(10^{-4}\). 
We terminate training on convergence,
when the loss on the validation set starts to increase,
or after at most~\num{500} iterations. 

For assessing the performance of repaired networks, we compare the accuracy and the 
Mean Absolute Error (MAE) between the predictions of the repaired network and the 
predictions of the original network on a large grid of inputs,
filtering out counterexamples.
For all networks, the filtered grid contains more than~\num{24} million points.

\subsubsection{CollisionDetection}
The CollisionDetection dataset~\citep{Ehlers2017} is introduced 
for evaluating neural network verifiers.
The task is to predict whether two particles collide based on their
relative position, speed, and turning angles. 
The training set of~\num{7000} instances and the test set of~\num{3000} instances are 
obtained from simulating particle dynamics for randomly sampled initial configurations.
We train a small fully-connected neural network with~\num{20} neurons on this dataset.
The full architecture is given in Table~\ref{tab:network-architectures}.

Similarly to MNIST, we repair the adversarial robustness of this network for 
100 non-overlapping groups of ten randomly sampled inputs from the training set.
Here we also include inputs that do not violate the specification
to gather a sufficient number of groups.
Each robustness property has a radius of~\num{0.05}. 

The CollisionDetection network is trained for~\num{1000} iterations using Adam~\citep{KingmaBa2015}
with a learning rate of~\num{0.01}.
Repair uses Adam with a learning rate of~\num{0.001}, terminating training on convergence
or when reaching~\num{5000} iterations.

\subsubsection{Integer Dataset RMIs}
Learned index structures replace index structures, such as B-trees, with 
machine learning models~\citep{KraskaBeutelChiEtAl2018}.
\citet{TanZhuGuo2021}~identify these models as prime candidates for neural network
verification and repair due to the strict requirements of their domain 
and the small size of the models.
We use \emph{Recursive Model Indices} (RMIs)~\citep{KraskaBeutelChiEtAl2018}
in our experiments.
The task of an RMI is to resolve a key to a position in a sorted sequence.

We build datasets, RMIs and specifications following~\citet{TanZhuGuo2021},
with the exception that we create models of two sizes.
While~\citet{TanZhuGuo2021} build one RMI with a second-stage size of ten,
we build ten RMIs with a second-stage size of ten and~\num{50} RMIs 
with a second-stage size of~\num{304}.
Each RMI is constructed for a different dataset.
We create models of two second-stage sizes because the smaller size does not yield 
unsafe first-stage models, while the larger second-stage size does not yield 
unsafe second-stage models.
However, we want to repair models of both stages.
Therefore, when repairing first-stage models in Section~\ref{sec:experiments-falsifiers},
we repair the first-stage models of the RMIs with second-stage size~\num{304}.
In Section~\ref{sec:experiments-linear}, we repair the second-stage models
of the RMIs with second-stage size ten.

Each dataset is a randomly generated integer dataset
consisting of a sorted sequence of \num{190000} integers.
The integers are randomly sampled from a uniform distribution with the range~\(\left[0, 10^6\right]\).
The task is to predict the index of an integer (key) in the sorted sequence.

We build an RMI for each dataset.
Each RMI consists of two stages.
The first stage contains one neural network. 
The second stage contains several linear regression models. 
In our case, the second stage contains either ten or~\num{304} models.
Each dataset is first split into several disjoint blocks, 
one for each second-stage model.
Now, the first-stage network is trained to predict the block an integer key belongs to.
The purpose of this model is to select a model from the second stage.
Each model of the second stage is responsible for resolving the keys in a block 
to the position of the key in the sorted sequence.
The architectures of the models are given in Table~\ref{tab:network-architectures}.

We train the first-stage model to minimise the Mean Square Error (MSE) between the 
model predictions and the true blocks.
Training uses Adam~\citep{KingmaBa2015} with a learning
rate of~\num{0.01} and a mini-batch size of~\num{512}.
For an RMI with a second-stage size of ten, we train for one epoch.
For the larger second-stage size of~\num{304}, we train for six epochs.

Minimising the MSE between the positions a second-stage model predicts 
and the true positions can be solved analytically.
We use the analytic solution for training the second-stage models.
In Section~\ref{sec:experiments-linear}, we compare with
Ouroboros~\citep{TanZhuGuo2021} that also uses the analytic solution. 
SpecRepair~\citep{BauerMarquartBoetiusLeueEtAl2021} can not make use of the analytic solution.
Instead, it repairs second-stage models using gradient descent 
with a learning rate of~\(10^{-13}\), running for~\num{150} iterations.

The specifications for the RMIs are error bounds on the predictions of each model.
For a first-stage neural network, the specification is that it may not deviate 
by more than one valid block from the true block. 
The specification for a second-stage model consists of one property for each key 
in the target block and one for all other keys that the first stage
assigns to the second-stage model. 
The property for a key~\(k_i\) specifies that the prediction for all keys 
between the previous key~\(k_{i-1}\) and the next key~\(k_{i+1}\) in the dataset
may not deviate by more than~\(\varepsilon\) from the position for~\(k_i\).
We use two sets of specifications, one with~\(\varepsilon = 100\)
and one with~\(\varepsilon = 150\).
According to~\citet{TanZhuGuo2021}, the specifications express 
a guaranteed error bound for looking up both existing and non-existing keys. 
The formal definitions of the specifications are given in Appendix~\ref{sec:specs-rmi}.

As the original Ouroboros implementation is not publicly available, we
reimplement Ouroboros, including creating integer dataset RMIs.
To roughly estimate whether our reimplementation is faithful, 
we can examine the average size of the specifications for
the second-stage models for the RMIs with second-stage size ten.
As these specifications include keys that are wrongly assigned to 
a second-stage model, they can serve for quantifying the accuracy of
the first-stage models.
The specifications that we obtain for these models have a similar average
size as reported by~\citet{TanZhuGuo2021} (\num{19426} properties).
This indicates that our reimplementation is faithful.

\subsection{Repair Algorithms}
Except for Section~\ref{sec:experiments-linear}, we exclusively use
the SpecRepair counterexample-removal 
algorithm~\citep{BauerMarquartBoetiusLeueEtAl2021} 
in our experiments.
In Section~\ref{sec:experiments-linear}, 
we also use the Ouroboros~\citep{TanZhuGuo2021} repair algorithm
and the novel Quadratic Programming repair algorithm.
Additionally, we leverage different falsifiers for repair.
This section introduces these tools.

Both Ouroboros and SpecRepair are counterexample-guided repair algorithms.
Ouroboros performs repair by augmenting the training set with counterexamples
and retraining the linear regression models using an analytic solution.
SpecRepair uses the~\(L_1\) penalty function method~\citep{NocedalWright2006}.
We use SpecRepair with a decreased initial penalty weight of~\(2^{-4}\)
and a satisfaction constant of~\(10^{-4}\).
We do not limit the number of repair steps a repair algorithm may perform,
except for Section~\ref{sec:experiments-linear}.
Here, we perform at most two repair steps for SpecRepair. 
For Ouroboros, we perform up to five repair steps following~\citet{TanZhuGuo2021}.

The Quadratic Programming repair algorithm from 
Section~\ref{sec:experiments-linear} is exact. 
That is, we obtain an infeasible problem if and only if the 
linear regression model can not satisfy
the specification and otherwise obtain the optimal repaired regression model.
To mitigate floating point issues, 
we require the satisfaction function to be at least~\(10^{-2}\) 
in Equation~\eqref{eqn:r} instead of requiring it to be just non-negative.
That corresponds to applying a satisfaction constant 
as in SpecRepair~\citep{BauerMarquartBoetiusLeueEtAl2021}.

We run SpecAttack~\citet{BauerMarquartBoetiusLeueEtAl2021} using 
Sequential Least SQuares Programming (SLSQP) as 
network gradients are available.
Following~\citet{CarliniWagner2017}, we run 
BIM~\citep{KurakinGoodfellowBengio2016,MadryMakelovSchmidtEtAl2018}
with Adam~\citep{KingmaBa2015} as optimiser.
BIM performs local optimisation ten times from different random starting points.

\subsection{Implementation and Hardware}\label{sec:experiment-design-impl-hardware}
We build upon SpecRepair~\citep{BauerMarquartBoetiusLeueEtAl2021} for our experiments,
leveraging the modified ERAN.\@
SpecRepair and ERAN are implemented in Python.
SpecRepair is based on PyTorch~\citep{PaszkeGrossMassaEtAl2019}. 
For repairing linear regression models, we also use an 
ERAN-based Python reimplementation of Ouroboros~\citep{TanZhuGuo2021}.
The original Ouroboros implementation is not publicly available.
The quadratic programming repair algorithm for linear regression models
is implemented in Python and leverages Gurobi~\citep{GurobiOptimization2021}.

All experiments were conducted on Ubuntu 2022.04.1 LTS machines using Python 3.8.
The ACAS Xu, CollisionDetection and Integer Dataset RMI experiments were run on a 
compute server with an Intel Xeon E5--2580 v4 2.4GHz CPU (28 cores) and 1008GB of memory.
The MNIST experiments were run on a GPU compute server 
with an AMD Ryzen Threadripper 3960X 24-Core Processor and 252GB of memory, 
utilising an NVIDIA RTX A6000 GPU with 48GB of memory.

We limit the execution time for repairing each ACAS Xu network and 
each MNIST specification to three hours.
For CollisionDetection and the Integer Dataset RMIs, we use a shorter timeout of one hour.
Except for ACAS Xu, whenever we report runtimes, we repeat all experiments ten times 
and report the median runtime from these runs.
This way, we obtain more accurate runtime measurements that are necessary
for interpreting runtime differences below one minute.
For ACAS Xu, the runtime differences are sufficiently large for all but one network,
so that we can faithfully compare different counterexample searchers 
without repeating the experiments.

\section{Specifications}\label{sec:specs}
In this section, we formally define the specifications used throughout this paper.

\subsection[L infinity Adversarial Robustness]{\(\mathbf{L}_{\boldsymbol{\infty}}\) Adversarial Robustness}\label{sec:specs-linf-robustness}
Adversarial robustness is a specification for classifiers capturing that small
perturbations of an input may not change the classification.
For~\(L_\infty\) adversarial robustness, the input set is
an \(L_\infty\) ball (a hyper-rectangle).

Assume the input space has~\(n\) dimensions and there 
are~\(m\) classes.
Furthermore, assume the classifier produces a score for each class
so that the classifier has~\(m\) outputs.
Also, assume the classification is the class with the largest score.
Let~\(\mathcal{D} \subset \Reals^n \times \{1, \ldots, m\}\) be the set of labelled
inputs for which we want to specify adversarial robustness.
Then, the~\(L_\infty\) adversarial robustness specification with radius~\(\varepsilon\) is
\begin{subequations}
  \begin{align}
    \Phi &= {\{\varphi(\vec{x}, c) \mid (\vec{x}, c) \in \mathcal{D}\}} \\
    \varphi(\vec{x}, c) &= \left(
      \left\{
        \vec{x}' \mid 
        \vec{x}' \in \Reals^n, 
        {\lVert \vec{x}' - \vec{x} \rVert}_\infty \leq \varepsilon 
      \right\},
      \left\{
        \vec{y} \;\left\vert\;
        \vec{y} \in \Reals^m, 
        \bigwedge_{i=1}^m \vec{y}_i \leq \vec{y}_c
      \right.\right\}
    \right).
  \end{align}
\end{subequations}
The SpecRepair satisfaction function~\citep{BauerMarquartBoetiusLeueEtAl2021} for a
property~\(\varphi(\vec{x}, c)\) is
\begin{equation}
  \fSAT(\vec{y}) = \min_{i=1}^m \vec{y}_c - \vec{y}_i.\label{eqn:spec-robust-fsat-specrepair}
\end{equation}
Equivalently, we can split up each property into several properties with just one
linear constraint, yielding a linear specification.
We describe this in Section~\ref{sec:specs-rmi}.

\subsection[ACAS Xu phi 2]{ACAS Xu \(\boldsymbol{\phi}_\mathbf{2}\)}\label{sec:specs-acasxu}
This safety specification consists of a single property.
The property~\(\phi_2\) of~\citet{KatzBarrettDillEtAl2017} is
\begin{equation}
  \phi_2 = \left(
    [55947.691, \infty] \times \Reals^2 \times [1145, \infty] \times [-\infty, 60], 
    \left\{
      \vec{y} \;\left\vert\; \vec{y} \in \Reals^5, \bigvee_{i=1}^5 \vec{y}_1 \leq \vec{y}_i
    \right.\right\}
  \right).
\end{equation}
The output set expresses that Clear-of-Conflict is not the maximal score,
or, in other words, is not least advised.
The input set of this property is unbounded, but each ACAS Xu network has a
bounded input domain.
Intersected with one of the input domains, we obtain a closed input set for the property.
The SpecRepair satisfaction function 
for~\(\phi_2\) is
\begin{equation}
  \fSAT(\vec{y}) = \max_{i=1}^m \vec{y}_i - \vec{y}_1.
\end{equation}

\subsection{Integer Dataset RMI Error Bounds}\label{sec:specs-rmi}
For an RMI, the specification of a first-stage model is that it may not deviate
by more than one valid block from the true block a key resides in.
Let~\([l_1, u_1], [l_2, u_2], \ldots, [l_K, u_K]\) be the blocks of the RMI,
where~\(K \in \{10, 304\}\) is the number of blocks.
Then, the specification of the first-stage model is
\begin{subequations}
  \begin{align}
    \Phi &= {\{\varphi_i\}}_{i=1}^{K} \\
    \varphi_i &= \left(
      [l_i, u_i], 
      [\max(1, i-1), \min(i+1, K)]
    \right)\quad \forall i \in \{1, \ldots, K\}.
  \end{align}
\end{subequations}

The specification of a second-stage model contains one property for each key~\(k_i\) 
that is in the model's target block or is assigned to the model by the first-stage model.
Let~\(\mathcal{K} \subset \Nats\) be the indices of keys that are in the model's block
or assigned to it.
Each property expresses that the predictions for all keys between the 
previous key~\(k_{i-1}\) and the next key~\(k_{i+1}\) deviate by at most~\(\varepsilon \in \{100, 150\}\) 
from the true position~\(p_i\) of~\(k_i\).
When there is no previous or next key in the dataset, we use the key itself as bound.
Therefore,~\(k_{0} = k_{1}\) and~\(k_{\num{190001}} = k_{\num{190000}}\).
Now, the second-stage specification is
\begin{subequations}
  \begin{align}
    \Phi &= {\{\varphi_i\}}_{i \in \mathcal{K}} \\
    \varphi_i &= \left([\min(k_{i-1}+1, k_i), \max(k_{i+1}-1, k_i)], [p_i - \varepsilon, p_i + \varepsilon]\right).
  \end{align}
\end{subequations}

The output sets of these properties are hyper-rectangles.
In the SpecRepair property format, one-dimensional hyper-rectangles
correspond to a conjunction of two linear constraints.
A conjunction of multiple linear constraints does not yield an affine satisfaction function,
as SpecRepair uses a minimum to encode conjunctions.
This is illustrated by Equation~\eqref{eqn:spec-robust-fsat-specrepair}.
To obtain a linear specification, we can split each property into two properties, such that each
property only has one linear constraint. 
Using this alternative formulation, the specification of a second-stage model is
\begin{subequations}\label{eqn:rmi-second-stage-spec-2}
  \begin{align}
    \Phi &= {\{\varphi_i\}}_{i \in \mathcal{K}} \cup {\{\psi_i\}}_{i \in \mathcal{K}}\\
    \varphi_i &= \left([\min(k_{i-1}+1, k_i), \max(k_{i+1}-1, k_i)], \{\vec{y} \mid \vec{y} \in \Reals, \vec{y} \geq p_i - \varepsilon\}\right) \\
    \psi_i &= \left([\min(k_{i-1}+1, k_i), \max(k_{i+1}-1, k_i)], \{\vec{y} \mid \vec{y} \in \Reals, \vec{y} \leq p_i + \varepsilon\}\right).
  \end{align}
\end{subequations}
This formulation yields the affine SpecRepair satisfaction 
functions~\(\fSAT(\vec{y}) = \vec{y} - p_i - \varepsilon\) for the property~\(\varphi_i\)
and~\(\fSAT(\vec{y}) = p_i + \varepsilon - \vec{y}\) for~\(\psi_i\).
As the input set of each property is a hyper-rectangle,~\(\Phi\)
from Equation~\eqref{eqn:rmi-second-stage-spec-2} is a linear
specification.

\section{Additional Experimental Results}\label{sec:experiments-extra}
In this section, we report additional experimental results.
This includes additional results on using falsifiers for repair
and an overview of the success rates and repaired network performance
when using different verifiers and falsifiers for repair.
For comparison with earlier work, we report our full ACAS Xu results
in Section~\ref{sec:experiments-extra-acasxu}.

\textbf{A Note on Failing Repairs}\quad
We witness several failing repairs in our experiments. 
These are either due to timeout or due to failing counterexample-removal.
There are no indications of non-termination regarding Algorithm~\ref{algo:cgr}
itself in these failing repairs.
In other words, we do not observe exceedingly high repair step counts.
This holds true both for the optimal verifier, for which termination remains an open question,
and the early-exit verifier, for which we disprove termination in 
Section~\ref{sec:theory-early-exit-verifiers}.

\subsection{Using Falsifiers for Repair}\label{sec:experiments-extra-falsifiers}
\begin{figure}[tb]
  \centering
  \begin{subfigure}{0.5\linewidth}
    \centering
    \includegraphics[width=0.95\linewidth]{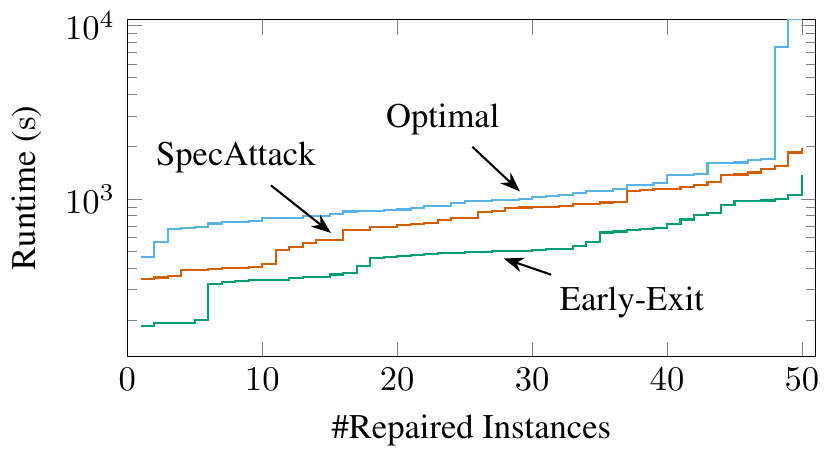}
    \caption{MNIST}
  \end{subfigure}%
  \begin{subfigure}{0.5\linewidth}
    \centering
    \includegraphics[width=0.95\linewidth]{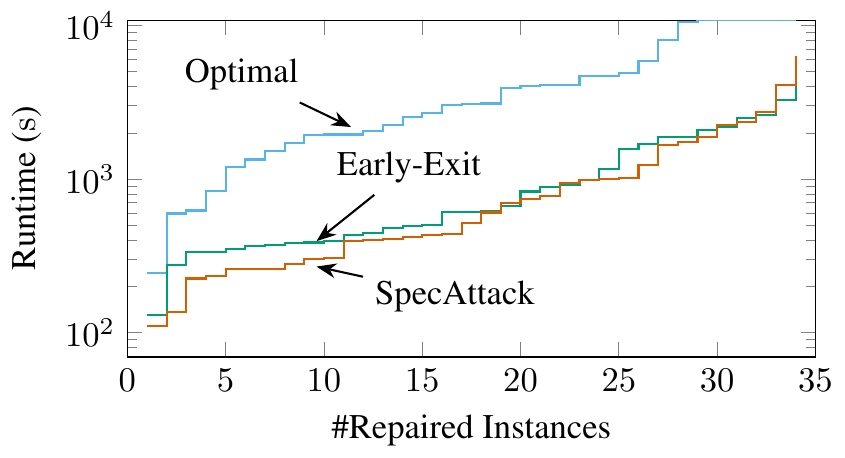}
    \caption{ACAS Xu}
  \end{subfigure}
  \caption[Repair using Falsifiers]{%
    Repair using Falsifiers.
    We plot the number of repaired instances that individually
    require less than a certain runtime.
    We plot this for repair using
    SpecAttack~\includegraphics{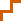}, 
    only the optimal verifier~\includegraphics{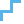}
    and only the early-exit
    verifier~\includegraphics{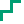}.
    Both experiments use a timeout of three hours.
    Runtimes are given on a logarithmic scale.
  }\label{fig:falsifiers-cactus}
\end{figure}
In this section, we report additional results on repair using falsifiers.
To study the advantages of falsifiers for repair, we repair an MNIST network and the 
ACAS Xu networks using the SpecAttack~\citep{BauerMarquartBoetiusLeueEtAl2021} and
BIM~\citep{KurakinGoodfellowBengio2016,MadryMakelovSchmidtEtAl2018} falsifiers. 
We outline the approach of the BIM falsifier 
in Section~\ref{sec:nn-verif}.
We start repair by searching counterexamples using one of the falsifiers.
Only when the falsifier fails to produce further counterexamples we
turn to the early-exit verifier. 
Ideally, we would want that the verifier is invoked only once to prove 
specification satisfaction.
Practically, often several additional repair steps have to be performed using the verifier. 

For ACAS Xu, we observe that BIM generally fails to find counterexamples.
Therefore, we only report using SpecAttack for ACAS Xu.
For the small CollisionDetection and Integer Dataset networks, the verifier is
already comparably fast, so neither BIM nor SpecAttack can provide a runtime advantage. 
We also evaluated combining falsifiers with the optimal verifier, but this does not
improve upon using the early-exit verifier. 

\textbf{SpecAttack}\quad
Figure~\ref{fig:falsifiers-cactus} summarises our results for repairing the MNIST network
and the ACAS Xu networks using SpecAttack.
For the MNIST network, using SpecAttack is inferior to using only the early-exit 
verifier. 
In our experiments, SpecAttack provides no significant runtime advantage for generating 
counterexamples over the early-exit verifier and tends to compute
counterexamples with a smaller violation.
SpecAttack's runtime scales well with the network size but exponentially in the input dimension.
Thus, it is not surprising that it provides no advantage for our 
MNIST network, which is tiny compared
to state-of-the-art image classification networks.

For ACAS Xu, we would expect that SpecAttack outperforms using only the early-exit verifier
more clearly than apparent from Figure~\ref{fig:falsifiers-cactus}.
Here, SpecAttack's runtime is an order of magnitude faster than the runtime of the early-exit verifier.
SpecAttack can also provide an advantage in repair steps in many cases.
However, at times using SpecAttack also increases the number of repair steps.
Additionally, SpecAttack sometimes makes the final invocations of the early-exit verifier 
more costly than when only the verifier is used.

\textbf{Additional Results on BIM}\quad
The results of our experiments using BIM on MNIST are summarised in 
Figure~\ref{fig:falsifiers-cactus-mnist}.
We already discussed that BIM can significantly accelerate repair. 
For MNIST, repair using BIM is the fastest method in \SI{70}{\percent} of the repair cases, 
compared to~\SI{26}{\percent} for only the early-exit verifier,~\SI{2}{\percent}~for SpecAttack
and~\SI{0}{\percent} for only the optimal verifier.
In~\SI{2}{\percent} of the cases, the runtime of the two best variants is within \num{30} seconds. 
%
The breakdown of which method is the fastest for each repair case shows that 
the picture is not as clear as we may wish it to be~---~BIM provides 
a significant runtime advantage in~\SI{70}{\percent} of the cases, 
but in~\SI{26}{\percent} of the cases using 
only the early-exit verifier is faster.

Our experiments using falsifiers demonstrate that they can give a substantial
runtime advantage to repair, but they also show that speeding up repair 
traces back to more intricate properties beyond just falsifier speed.
Understanding these properties better is a promising future research direction for designing
better falsifiers for repair.

\subsection{Success Rates and Repaired Network Performance}
\begin{table}[tb]
  \caption[Experiments: Success Rates]{%
    Experiments: Success Rates.
    The success rates when repairing the ACAS Xu, MNIST, and
    CollisionDetection networks and integer dataset RMI
    first-stage models
    using the different verifiers and falsifiers.
    For ACAS Xu, BIM is not included as it 
    is unable to discover counterexamples for these networks.
  }\label{tab:success-rates}
  \centering
  \begin{tabular}{lcccc}
                     & \multicolumn{4}{c}{\textbf{Success Rate}} \\
    \textbf{Dataset} & \textbf{Optimal}
                     & \textbf{Early-Exit}
                     & \textbf{BIM}
                     & \textbf{SpecAttack} \\
    \midrule
    ACAS Xu            & \SI{82.5}{\percent} & \SI{100.0}{\percent} 
                       & -- & \SI{100.0}{\percent} \\
    MNIST              & \SI{96.0}{\percent}   & \SI{100.0}{\percent} 
                       & \SI{100.0}{\percent} & \SI{100.0}{\percent} \\
    Integer Datasets   & \SI{92.0}{\percent}   & \SI{92.0}{\percent}
                       & \SI{94.0}{\percent} & \SI{90.0}{\percent} \\
    CollisionDetection & \SI{90.0}{\percent} & \SI{90.0}{\percent} 
                       & \SI{89.0}{\percent} & \SI{88.0}{\percent}
  \end{tabular}
\end{table}
\begin{table}
  \caption[Experiments: Median Accuracy]{%
    Experiments: Median Accuracy.
    The median repaired network accuracy when repairing the ACAS Xu, MNIST,
    and CollisionDetection networks and integer dataset RMI
    first-stage models
    using the different verifiers and falsifiers.
    We report the test set accuracy for MNIST and CollisionDetection
    and the training set accuracy for the integer dataset RMIs.
    For ACAS Xu, we report the accuracy for 
    recreating the predictions of the original network for
    a large grid of inputs, as described in 
    Appendix~\ref{sec:experiments-design-datasets-etc}.
    We report the median accuracy among the cases where 
    repair is successful for all verifiers and falsifiers.
    For ACAS Xu, BIM is not included as it 
    is unable to discover counterexamples for these networks.
  }\label{tab:median-accuracies}
  \centering
  \begin{tabular}{lcccc}
                     & \multicolumn{4}{c}{\textbf{Median Accuracy}} \\
    \textbf{Dataset} & \textbf{Optimal}
                     & \textbf{Early-Exit}
                     & \textbf{BIM}
                     & \textbf{SpecAttack} \\
    \midrule
    ACAS Xu            & \SI{99.6}{\percent} & \SI{99.6}{\percent} 
                       & -- & \SI{99.6}{\percent} \\
    MNIST              & \SI{97.4}{\percent} & \SI{97.5}{\percent}
                       & \SI{97.4}{\percent} & \SI{97.5}{\percent} \\
    Integer Datasets   & \SI{90.9}{\percent} & \SI{90.9}{\percent}
                       & \SI{90.0}{\percent} & \SI{91.0}{\percent} \\
    CollisionDetection & \SI{89.8}{\percent} & \SI{90.2}{\percent}
                       & \SI{90.0}{\percent} & \SI{89.9}{\percent}
  \end{tabular}
\end{table}
Table~\ref{tab:success-rates} summarises the success rates of
repairing MNIST, ACAS Xu, and CollisionDetection networks
and Integer Dataset RMI first-stage models using different 
verifiers and falsifiers.
For the large MNIST and ACAS Xu networks, the early-exit verifier
enables repair in some cases where repair using the optimal verifier
fails due to timeout.
Regarding the use of falsifiers, there are minor variations for 
the Integer Dataset RMIs and CollisionDetection. 
These differences are due to failing repairs.
Here, the counterexample removal procedure is unable to remove
all counterexamples provided by, for example, the optimal verifier,
while it succeeds for another set of counterexamples.

Table~\ref{tab:median-accuracies} summarises the performance of the 
repaired networks.
We only observe minimal variations regarding the performance. 
Using the early-exit verifier slightly outperforms the optimal verifiers
on MNIST and CollisionDetection. 
The impact of BIM and SpecAttack is inconsistent across datasets.
We recommend fine-tuning the initial penalty weight to the counterexample
violation magnitude instead of using
a different counterexample searcher to increase repaired network performance.

\subsection{ACAS Xu}\label{sec:experiments-extra-acasxu}
\begin{table}[tb]
    \caption[Detailed ACAS Xu Results]{%
      Detailed ACAS Xu Results.
      Opt.\ and E.E denote repair using the optimal and the 
      early-exit verifier, respectively.
      Sp.A.\ denotes repair using SpecAttack and the early-exit verifier.
      The symbol~\exsuccess{} denotes successful repair, 
      while~\extimeout{} denotes timeout.
      Both accuracy and Mean Absolute Error (MAE) compare the
      predictions of the repaired network with the predictions
      of the initial faulty network for a large grid
      of inputs.
      More details on this are provided in 
      Appendix~\ref{sec:experiments-design-datasets-etc}.
    }\label{tab:acasxu-full}
    \centering
\begin{tabular}{llccccccccc}
         &       & \multicolumn{3}{c}{\textbf{Status}} & \multicolumn{3}{c}{\textbf{Accuracy}} & \multicolumn{3}{c}{\textbf{MAE}} \\
\midrule
\textbf{Spec.} & \textbf{Model} 
& \textbf{Opt.} & \textbf{E.E.} & \textbf{Sp.A.} 
& \textbf{Opt.} & \textbf{E.E.} & \textbf{Sp.A.} 
& \textbf{Opt.} & \textbf{E.E.} & \textbf{Sp.A.} \\
\midrule
\(\phi_2\) & \(N_{2,1}\) & \exsuccess & \exsuccess & \exsuccess & \SI{99.6}{\percent} & \SI{99.4}{\percent} & \SI{99.6}{\percent} & \num{  0.10} & \num{  0.15} & \num{  0.11} \\
\(\phi_2\) & \(N_{2,2}\) & \extimeout & \exsuccess & \exsuccess &  --  & \SI{99.2}{\percent} & \SI{99.1}{\percent} &    --  & \num{  0.18} & \num{  0.22} \\
\(\phi_2\) & \(N_{2,3}\) & \exsuccess & \exsuccess & \exsuccess & \SI{99.7}{\percent} & \SI{99.7}{\percent} & \SI{99.7}{\percent} & \num{  0.09} & \num{  0.10} & \num{  0.11} \\
\(\phi_2\) & \(N_{2,4}\) & \exsuccess & \exsuccess & \exsuccess & \SI{99.7}{\percent} & \SI{99.5}{\percent} & \SI{99.8}{\percent} & \num{  0.08} & \num{  0.11} & \num{  0.08} \\
\(\phi_2\) & \(N_{2,5}\) & \extimeout & \exsuccess & \exsuccess &  --  & \SI{99.4}{\percent} & \SI{99.3}{\percent} &    --  & \num{  0.11} & \num{  0.11} \\
\(\phi_2\) & \(N_{2,6}\) & \exsuccess & \exsuccess & \exsuccess & \SI{99.6}{\percent} & \SI{99.5}{\percent} & \SI{99.5}{\percent} & \num{  0.12} & \num{  0.11} & \num{  0.10} \\
\(\phi_2\) & \(N_{2,7}\) & \extimeout & \exsuccess & \exsuccess &  --  & \SI{99.5}{\percent} & \SI{14.5}{\percent} &    --  & \num{  0.17} & \num{  0.16} \\
\(\phi_2\) & \(N_{2,8}\) & \exsuccess & \exsuccess & \exsuccess & \SI{99.6}{\percent} & \SI{99.6}{\percent} & \SI{99.7}{\percent} & \num{  0.12} & \num{  0.15} & \num{  0.13} \\
\(\phi_2\) & \(N_{2,9}\) & \exsuccess & \exsuccess & \exsuccess & \SI{99.9}{\percent} & \SI{99.8}{\percent} & \SI{99.8}{\percent} & \num{  0.17} & \num{  0.14} & \num{  0.14} \\
\(\phi_2\) & \(N_{3,1}\) & \exsuccess & \exsuccess & \exsuccess & \SI{99.0}{\percent} & \SI{99.3}{\percent} & \SI{99.4}{\percent} & \num{  0.22} & \num{  0.14} & \num{  0.17} \\
\(\phi_2\) & \(N_{3,2}\) & \exsuccess & \exsuccess & \exsuccess & \SI{99.9}{\percent} & \SI{99.8}{\percent} & \SI{99.8}{\percent} & \num{  0.08} & \num{  0.09} & \num{  0.13} \\
\(\phi_2\) & \(N_{3,4}\) & \exsuccess & \exsuccess & \exsuccess & \SI{99.6}{\percent} & \SI{99.6}{\percent} & \SI{99.6}{\percent} & \num{  0.11} & \num{  0.13} & \num{  0.08} \\
\(\phi_2\) & \(N_{3,5}\) & \exsuccess & \exsuccess & \exsuccess & \SI{99.5}{\percent} & \SI{99.5}{\percent} & \SI{99.5}{\percent} & \num{  0.08} & \num{  0.10} & \num{  0.08} \\
\(\phi_2\) & \(N_{3,6}\) & \exsuccess & \exsuccess & \exsuccess & \SI{99.8}{\percent} & \SI{99.8}{\percent} & \SI{99.7}{\percent} & \num{  0.05} & \num{  0.06} & \num{  0.06} \\
\(\phi_2\) & \(N_{3,7}\) & \exsuccess & \exsuccess & \exsuccess & \SI{99.6}{\percent} & \SI{99.6}{\percent} & \SI{99.7}{\percent} & \num{  0.09} & \num{  0.08} & \num{  0.08} \\
\(\phi_2\) & \(N_{3,8}\) & \exsuccess & \exsuccess & \exsuccess & \SI{99.6}{\percent} & \SI{99.6}{\percent} & \SI{99.7}{\percent} & \num{  0.13} & \num{  0.10} & \num{  0.11} \\
\(\phi_2\) & \(N_{3,9}\) & \extimeout & \exsuccess & \exsuccess &  --  & \SI{97.5}{\percent} & \SI{97.5}{\percent} &    --  & \num{  0.19} & \num{  0.18} \\
\(\phi_2\) & \(N_{4,1}\) & \exsuccess & \exsuccess & \exsuccess & \SI{99.8}{\percent} & \SI{99.8}{\percent} & \SI{99.8}{\percent} & \num{  0.10} & \num{  0.07} & \num{  0.10} \\
\(\phi_2\) & \(N_{4,3}\) & \exsuccess & \exsuccess & \exsuccess & \SI{99.6}{\percent} & \SI{99.6}{\percent} & \SI{99.5}{\percent} & \num{  0.09} & \num{  0.08} & \num{  0.13} \\
\(\phi_2\) & \(N_{4,4}\) & \exsuccess & \exsuccess & \exsuccess & \SI{99.7}{\percent} & \SI{99.6}{\percent} & \SI{99.7}{\percent} & \num{  0.11} & \num{  0.15} & \num{  0.09} \\
\(\phi_2\) & \(N_{4,5}\) & \exsuccess & \exsuccess & \exsuccess & \SI{99.6}{\percent} & \SI{99.5}{\percent} & \SI{99.5}{\percent} & \num{  0.09} & \num{  0.07} & \num{  0.12} \\
\(\phi_2\) & \(N_{4,6}\) & \extimeout & \exsuccess & \exsuccess &  --  & \SI{99.5}{\percent} & \SI{99.6}{\percent} &    --  & \num{  0.13} & \num{  0.11} \\
\(\phi_2\) & \(N_{4,7}\) & \extimeout & \exsuccess & \exsuccess &  --  & \SI{99.2}{\percent} & \SI{99.2}{\percent} &    --  & \num{  0.17} & \num{  0.18} \\
\(\phi_2\) & \(N_{4,8}\) & \exsuccess & \exsuccess & \exsuccess & \SI{99.4}{\percent} & \SI{99.3}{\percent} & \SI{99.5}{\percent} & \num{  0.08} & \num{  0.08} & \num{  0.09} \\
\(\phi_2\) & \(N_{4,9}\) & \exsuccess & \exsuccess & \exsuccess & \SI{96.5}{\percent} & \SI{96.4}{\percent} & \SI{96.5}{\percent} & \num{  0.16} & \num{  0.14} & \num{  0.13} \\
\(\phi_2\) & \(N_{5,1}\) & \exsuccess & \exsuccess & \exsuccess & \SI{99.7}{\percent} & \SI{99.5}{\percent} & \SI{99.6}{\percent} & \num{  0.12} & \num{  0.13} & \num{  0.09} \\
\(\phi_2\) & \(N_{5,2}\) & \exsuccess & \exsuccess & \exsuccess & \SI{99.7}{\percent} & \SI{99.7}{\percent} & \SI{99.7}{\percent} & \num{  0.08} & \num{  0.07} & \num{  0.08} \\
\(\phi_2\) & \(N_{5,3}\) & \exsuccess & \exsuccess & \exsuccess & \SI{99.9}{\percent} & \SI{99.9}{\percent} & \SI{99.9}{\percent} & \num{  0.04} & \num{  0.04} & \num{  0.04} \\
\(\phi_2\) & \(N_{5,4}\) & \exsuccess & \exsuccess & \exsuccess & \SI{99.7}{\percent} & \SI{99.7}{\percent} & \SI{99.7}{\percent} & \num{  0.10} & \num{  0.07} & \num{  0.10} \\
\(\phi_2\) & \(N_{5,5}\) & \exsuccess & \exsuccess & \exsuccess & \SI{99.8}{\percent} & \SI{99.7}{\percent} & \SI{99.8}{\percent} & \num{  0.10} & \num{  0.13} & \num{  0.11} \\
\(\phi_2\) & \(N_{5,6}\) & \exsuccess & \exsuccess & \exsuccess & \SI{99.6}{\percent} & \SI{99.5}{\percent} & \SI{99.5}{\percent} & \num{  0.13} & \num{  0.11} & \num{  0.11} \\
\(\phi_2\) & \(N_{5,7}\) & \exsuccess & \exsuccess & \exsuccess & \SI{99.3}{\percent} & \SI{99.3}{\percent} & \SI{99.5}{\percent} & \num{  0.17} & \num{  0.14} & \num{  0.11} \\
\(\phi_2\) & \(N_{5,8}\) & \exsuccess & \exsuccess & \exsuccess & \SI{99.6}{\percent} & \SI{99.4}{\percent} & \SI{99.5}{\percent} & \num{  0.13} & \num{  0.12} & \num{  0.13} \\
\(\phi_2\) & \(N_{5,9}\) & \exsuccess & \exsuccess & \exsuccess & \SI{99.0}{\percent} & \SI{98.9}{\percent} & \SI{99.1}{\percent} & \num{  0.14} & \num{  0.13} & \num{  0.09} \\
\midrule
           &             & 28 & 34 & 34 & \SI{99.6}{\percent} & \SI{99.5}{\percent} & \SI{99.6}{\percent} & \num{  0.10} & \num{  0.11} & \num{  0.11} \\
           &             & \multicolumn{3}{c}{success frequency} & \multicolumn{3}{c}{median} & \multicolumn{3}{c}{median}
\end{tabular}
\end{table}
For comparison with the earlier work
of~\citet{BauerMarquartBoetiusLeueEtAl2021},
we report our detailed ACAS Xu results in Table~\ref{tab:acasxu-full}.
Due to improvements in the interaction with the verifier, we 
are successful more frequently than any method evaluated 
by~\citet{BauerMarquartBoetiusLeueEtAl2021}. 
At the same time, we maintain the level of repaired network performance.


\end{document}